\documentclass[sigconf]{aamas} 
\usepackage{balance} % for balancing columns on the final page

\setcopyright{ifaamas}
\acmConference[AAMAS '22]{Proc.\@ of the 21st International Conference
on Autonomous Agents and Multiagent Systems (AAMAS 2022)}{May 9--13, 2022}
{Online}{P.~Faliszewski, V.~Mascardi, C.~Pelachaud,
M.E.~Taylor (eds.)}
\copyrightyear{2022}
\acmYear{2022}
\acmDOI{}
\acmPrice{}
\acmISBN{}

\acmSubmissionID{541}

\title{Unbiased Asymmetric Reinforcement Learning\\ under Partial Observability}

\author{Andrea Baisero}
\affiliation{
  \institution{Northeastern University}
  \city{Boston}
  \state{Massachusetts}
  \country{USA}}
\email{baisero.a@northeastern.edu}

\author{Christopher Amato}
\affiliation{
  \institution{Northeastern University}
  \city{Boston}
  \state{Massachusetts}
  \country{USA}}
\email{c.amato@northeastern.edu}

\begin{abstract}
  In partially observable reinforcement learning, offline training gives access to
  latent information which is not available during online training and/or
  execution, such as the system state.  Asymmetric actor-critic methods exploit
  such information by training a history-based policy via a state-based critic.
  However, many asymmetric methods lack theoretical foundation, and are only
  evaluated on limited domains.
  We examine the theory of asymmetric actor-critic methods which use
  state-based critics, and expose fundamental issues which undermine the
  validity of a common variant, and limit its ability to address partial
  observability.  We propose an unbiased asymmetric actor-critic variant which
  is able to exploit state information while remaining theoretically sound,
  maintaining the validity of the policy gradient theorem, and introducing no
  bias and relatively low variance into the training process.
  An empirical evaluation performed on domains which exhibit significant
  partial observability confirms our analysis, demonstrating that unbiased
  asymmetric actor-critic converges to better policies and/or faster than
  symmetric and biased asymmetric baselines.
\end{abstract}


\usepackage{amsthm}
\newtheorem{theorem}{Theorem}[section]

\usepackage{algorithm}
\usepackage[noend]{algorithmic}

\usepackage[all]{abaisero}
\usepackage{tikz}
\usetikzlibrary{calc}
\usetikzlibrary{matrix}
\usetikzlibrary{backgrounds}
\usetikzlibrary{shapes}
\usetikzlibrary{positioning}
\tikzstyle{rv.empty}=[circle]
\tikzstyle{rv}=[rv.empty,draw]
\tikzstyle{rv.fake}=[rv,dashed]
\tikzstyle{ov}=[rv,fill=lightgray!50]
\tikzstyle{ov.fake}=[ov,dashed]
\tikzstyle{arrow}=[->,thick]

\usepackage{subcaption}
\usepackage{cleveref}
\usepackage[inline]{enumitem}

\usepackage{todonotes}
         
\newcommand\xset{\mathcal{X}}
\newcommand\defeq{=}

\newcommand\ploss{\loss_\text{policy}}
\newcommand\closs{\loss_\text{critic}}
\newcommand\hloss{\loss_\text{neg-entropy}}
\newcommand\pparams{\theta}
\newcommand\cparams{\vartheta}

\newcommand\ach{\textbf{A2C}}
\newcommand\acs{\textbf{A2C-asym-s}}
\newcommand\achs{\textbf{A2C-asym-hs}}
\newcommand\achrtwo{\textbf{A2C-react-2}}
\newcommand\achrfour{\textbf{A2C-react-4}}

\newcommand\env[1]{\textbf{#1}}
\newcommand\heavenhell{\env{Heaven-Hell}}
\newcommand\heavenhellthree{\env{Heaven-Hell-3}}
\newcommand\heavenhellfour{\env{Heaven-Hell-4}}
\newcommand\shopping{\env{Shopping}}
\newcommand\shoppingfive{\env{Shopping-5}}
\newcommand\shoppingsix{\env{Shopping-6}}

\newcommand\memoryroomsseven{\env{Memory-Four-Rooms-7x7}}
\newcommand\memoryroomsnine{\env{Memory-Four-Rooms-9x9}}
\newcommand\carflag{\env{Car-Flag}}
\newcommand\cleaner{\env{Cleaner}}

\newcommand\good{\texttt{GOOD}}
\newcommand\bad{\texttt{BAD}}
\newcommand\GGG{\texttt{G}}
\newcommand\B{\texttt{B}}

\begin{document}

%%% The following commands remove the headers in your paper. For final 
%%% papers, these will be inserted during the pagination process.

\pagestyle{fancy}
\fancyhead{}

%%% The next command prints the information defined in the preamble.

\maketitle 

%%%%%%%%%%%%%%%%%%%%%%%%%%%%%%%%%%%%%%%%%%%%%%%%%%%%%%%%%%%%%%%%%%%%%%%%

\section{Introduction}\label{sec:intro}

\begin{figure}[t!]
  \centering
  \hfill
  \begin{subfigure}{.6\linewidth}
    \centering
    \fbox{\includegraphics[width=\linewidth]{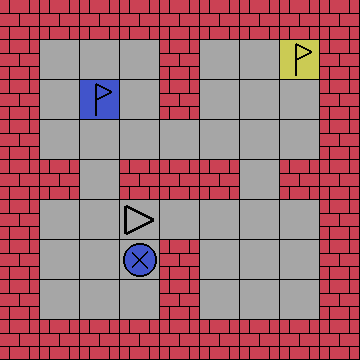}}
    \caption{State.}
  \end{subfigure}
  \hfill
  \begin{subfigure}{.3\linewidth}
    \centering
    \fbox{\includegraphics[width=\linewidth]{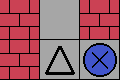}}
    \caption{Observation.}
  \end{subfigure}
  \hfill
  \caption{\memoryroomsnine, a procedurally generated navigation task which
  requires information-gathering and memorization.  The agent must avoid the \emph{bad} exit and reach the
\emph{good} exit, which is identifiable by the color of the \emph{beacon}.}\label{fig:frontpage}
\end{figure}

Partial observability is a key characteristic of many real-world reinforcement
learning (RL) control problems where the agent lacks access to the system
state, and is restricted to operate based on the observable past, a.k.a.\ the
\emph{history}.  Such control problems are commonly encoded as partially
observable Markov decision processes (POMDPs)~\cite{kaelbling_planning_1998},
which are the focus of a significant amount of research effort.
\emph{Offline learning}/\emph{online execution} is a common RL framework where
an agent is trained in a simulated \emph{offline} environment before operating
\emph{online}, which offers the possibility of using latent information not
generally available in online learning, e.g., the simulated system state, or
the state belief from the agent's
perspective~\cite{pinto_asymmetric_2017,karkus_particle_2018,jonschkowski_differentiable_2018,nguyen_belief-grounded_2020,warrington_robust_2021,chen_learning_2020}.

Offline learning methods are in principle able to exploit this privileged
information during training to achieve better online performance, so long as
the resulting agent does not use the latent information during online
execution.
Specifically, actor-critic
methods~\cite{sutton_policy_2000,konda_actor-critic_2000} are able to adopt
this approach via \emph{critic asymmetry}, where the policy and critic models
receive different
information~\cite{pinto_asymmetric_2017,foerster_counterfactual_2018,lowe_multi-agent_2017,li_robust_2019,wang_r-maddpg_2020,yang_cm3_2018,xiao_local_2021},
e.g., the history and latent state, respectively.
This is possible because the critic is merely a training construct, and is
not required or used by the agent to operate online.
By the very nature of actor-critic methods, critic models which are unable or
slow to learn accurate values act as a performance bottleneck on the policy.
Consequently, critic asymmetry is a powerful tool which, if carried out with
rigor, may provide significant benefits and bootstrap the agent's learning
performance.

Unfortunately, existing asymmetric methods use asymmetric information
heuristically, and demonstrate their validity only via empirical
experimentation on selected environments~\cite{pinto_asymmetric_2017,
foerster_counterfactual_2018,lowe_multi-agent_2017,li_robust_2019,wang_r-maddpg_2020,yang_cm3_2018,rashid_qmix_2018,mahajan_maven_2019,rashid_weighted_2020,nguyen_belief-grounded_2020,xiao_local_2021};
the lack of a sound theoretical foundation leaves uncertainties on whether
these methods are truly able to generalize to other environments, particularly
those wich feature higher degrees of partial observability (see
\Cref{fig:frontpage}).
In this work,
\begin{enumerate*}[label=(\alph*)]
  \item we analyze a standard variant of asymmetric actor-critic and expose
    analytical issues associated with the use of a state critic, namely that
    the state value function is generally ill-defined and/or causes learning
    bias;
  \item we prove an \emph{asymmetric policy gradient theorem} for partially
    observable control, an extension of the policy gradient theorem which
    explicitly uses latent state information;
  \item we propose a novel \emph{unbiased} asymmetric actor-critic method,
    which lacks the analytical issues of its \emph{biased} counterparts and is,
    to the best of our knowledge, the first of its kind to be theoretically
    sound;
  \item we validate our theoretical findings through empirical evaluations on
    environments which feature significant amounts of partial observability,
    and demonstrate the advantages of our unbiased variant over the symmetric and
    biased asymmetric baselines.
\end{enumerate*}

This work sets the stage for other asymmetric critic-based policy gradient
methods to exploit asymmetry in a principled manner, while learning under
partial observability.  Although we focus on \emph{advantage actor-critic}
(A2C), our method is easily extended to other critic-based learning methods
such as \emph{off-policy actor-critic}
\cite{degris_off-policy_2012,wang_sample_2017}, \emph{(deep) deterministic
policy gradient}~\cite{silver_deterministic_2014,lillicrap_continuous_2015},
and \emph{asynchronous actor-critic}~\cite{mnih_asynchronous_2016}.
Offline training is also the dominant paradigm in multi-agent RL, where many
asymmetric actor-critic methods could be similarly
improved~\cite{foerster_counterfactual_2018,lowe_multi-agent_2017,li_robust_2019,wang_r-maddpg_2020,yang_cm3_2018,rashid_qmix_2018,mahajan_maven_2019,rashid_weighted_2020,xiao_local_2021}.

\section{Related Work}

The use of latent information during offline training has been successfully
adopted in a variety of policy-based
methods~\cite{pinto_asymmetric_2017,foerster_counterfactual_2018,lowe_multi-agent_2017,yang_cm3_2018,li_robust_2019,wang_r-maddpg_2020,de_witt_deep_2021,warrington_robust_2021,xiao_local_2021}
and value-based
methods~\cite{rashid_qmix_2018,mahajan_maven_2019,rashid_weighted_2020,de_witt_deep_2021}.
Among the single-agent methods, \emph{asymmetric actor-critic for robot
learning}~\cite{pinto_asymmetric_2017} uses a reactive variant of DDPG with a
state-based critic to help address partial observability; belief-grounded
networks~\cite{nguyen_belief-grounded_2020} use a belief-reconstruction
auxiliary task to train history representations; and
Warrington et al.~\cite{warrington_robust_2021} and Chen et
al.~\cite{chen_learning_2020} use a fully observable agent trained offline on
latent state information to train a partially observable agent via imitation.

Asymmetric learning has also become popular in the multi-agent setting:
COMA~\cite{foerster_counterfactual_2018} uses reactive control and a shared
asymmetric critic which can receive either the joint observations of all agents
or the system state to solve cooperative tasks;
MADDPG~\cite{lowe_multi-agent_2017} and M3DDPG~\cite{li_robust_2019} use the
same form of asymmetry with individual asymmetric critics to solve
cooperative-competitive tasks;
R-MADDPG~\cite{wang_r-maddpg_2020} uses recurrent models to represent
non-reactive control, and the centralized critic uses the entire histories of
all agents;
CM3~\cite{yang_cm3_2018} uses a state critic for reactive control;
while ROLA~\cite{xiao_local_2021} trains centralized and local history/state
critics to estimate individual advantage values.
Asymmetry is also used in multi-agent value-based methods:
QMIX~\cite{rashid_qmix_2018}, MAVEN~\cite{mahajan_maven_2019}, and
WQMIX~\cite{rashid_weighted_2020} all train individual Q-models using a
centralized but factored Q-model, itself trained using state, joint histories,
and joint actions.

\section{Background}

In this section, we review background topics relevant to understand our work,
i.e., POMDPs, the RL graphical model, standard (symmetric) actor-critic, and
asymmetric actor-critic.

\paragraph{Notation}
We denote sets with calligraphy $\xset$, set elements with lowercase
$x\in\xset$, random variables (RVs) with uppercase $X$, and the set of
distributions over set $\xset$ as $\Delta\xset$.  Occasionally, we will need
absolute and/or relative time indices;  We use subscript $x_t$ to indicate
absolute time, and superscript $x\iter{k}$ to indicate the relative time of
variables, e.g., $x\iter{0}$ marks the beginning of a sequence happening at an
undetermined absolute time, and $x\iter{k}$ is the variable $k$ steps later.
We also use the bar notation to represent a sequence of superscripted variables
$\bar x = (x\iter{0}, x\iter{1}, x\iter{2}, \ldots)$.

\subsection{POMDPs}

A POMDP~\cite{kaelbling_planning_1998} is a discrete-time partially observable
control problem determined by a tuple $\langle \sset, \aset, \oset, T, O, R, \gamma \rangle$ consisting of: state, action and observation spaces
$\sset$, $\aset$, and $\oset$; transition function
$T\colon\sset\times\aset\to\Delta\sset$; observation function
$O\colon\sset\times\aset\times\sset\to\Delta\oset$; reward function
$R\colon\sset\times\aset\to\realset$; and discount factor $\gamma\in\left[0,
1\right]$.  The control goal is that of maximizing the expected discounted sum
of rewards $\Exp\left[ \sum_t \gamma^t R(S_t, A_t) \right]$, a.k.a.\ the
\emph{expected return}.

In the partially observable setting, the agent lacks access to the underlying
state, and actions are selected based on the observable \emph{history} $h$,
i.e., the sequences of past actions and observations.  We denote the space of
\emph{realizable}\footnote{Realizable histories and/or states have a non-zero
probability.} histories as $\hset\subseteq \left(\aset\times\oset\right)^*$,
and the space of \emph{realizable} histories of length $l$ as $\hset_l
\subseteq \left(\aset\times\oset\right)^l$.  Generally, an agent operating
under partial observability might have to consider the entire history to
achieve optimal behavior~\cite{singh_learning_1994}, i.e., its policy should
represent a mapping $\pi\colon\hset\to\Delta\aset$.  The \emph{belief-state}
$b\colon \hset\to\Delta\sset$ is the conditional distribution over states given
the observable history, i.e., $b(h) = \Pr(S\mid h)$, and a sufficient statistic
of the history for optimal control~\cite{kaelbling_planning_1998}.  We define
the history reward function as $\rfn(h, a) \defeq \Exp_{s\mid h}\left[ \rfn(s,
a) \right]$; from the agent's perspective, this is the reward function of the
decision process.  We denote the last observation in a history $h$ as $o_h$,
and say that an agent is \emph{reactive} if its policy
$\pi\colon\oset\to\Delta\aset$ only uses $o_h$ rather than the entire history.
A policy's history value function $\vpolicy\colon\hset\to\realset$ is the
expected return following a realizable history $h$,
\begin{equation}
  \vpolicy(h\iter{0}) = \Exp_{\bar s, \bar a\mid h\iter{0}}\left[
  \sum_{k=0}^\infty \gamma^k R(s\iter{k}, a\iter{k}) \right] \,,
  \label{eq:vh-raw}
\end{equation}
\noindent which supports an indirect recursive Bellman form,
\begin{align}
  \vpolicy(h) &= \sum_{a\in\aset} \pi(a; h) \qpolicy(h, a) \,, \label{eq:vh} \\
  \qpolicy(h, a) &= \rfn(h, a) + \gamma \Exp_{o\mid h, a}\left[ \vpolicy(hao)
  \right] \,. \label{eq:qha}
\end{align}

\subsection{The RL Graphical Model}\label{sec:graphmodel}

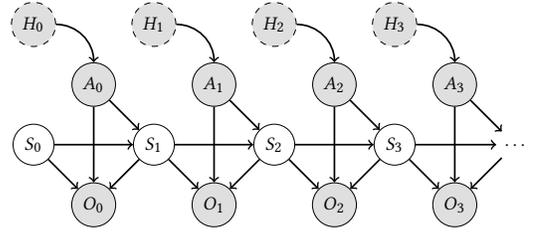
\begin{figure}
  \centering
  \scalebox{0.8}{
    \begin{tikzpicture}

  \node[rv] (S0) at (0, 0) {$S_0$};
  \node[rv] (S1) at (2, 0) {$S_1$};
  \node[rv] (S2) at (4, 0) {$S_2$};
  \node[rv] (S3) at (6, 0) {$S_3$};
  \node[rv.empty]     (S4) at (8, 0) {$\ldots$};

  \node[ov] (A0) at (1, 1) {$A_0$};
  \node[ov] (A1) at (3, 1) {$A_1$};
  \node[ov] (A2) at (5, 1) {$A_2$};
  \node[ov] (A3) at (7, 1) {$A_3$};

  \node[ov] (O0) at (1, -1) {$O_0$};
  \node[ov] (O1) at (3, -1) {$O_1$};
  \node[ov] (O2) at (5, -1) {$O_2$};
  \node[ov] (O3) at (7, -1) {$O_3$};

  \draw[arrow] (S0) -- (S1);
  \draw[arrow] (S1) -- (S2);
  \draw[arrow] (S2) -- (S3);
  \draw[arrow] (S3) -- (S4);

  \draw[arrow] (S0) -- (O0);
  \draw[arrow] (S1) -- (O1);
  \draw[arrow] (S2) -- (O2);
  \draw[arrow] (S3) -- (O3);

  \draw[arrow] (S1) -- (O0);
  \draw[arrow] (S2) -- (O1);
  \draw[arrow] (S3) -- (O2);
  \draw[arrow] (S4) -- (O3);

  \draw[arrow] (A0) -- (S1);
  \draw[arrow] (A1) -- (S2);
  \draw[arrow] (A2) -- (S3);
  \draw[arrow] (A3) -- (S4);

  \draw[arrow] (A0) -- (O0);
  \draw[arrow] (A1) -- (O1);
  \draw[arrow] (A2) -- (O2);
  \draw[arrow] (A3) -- (O3);

  \node[ov.fake] (H0) at (0, 2) {$H_0$};
  \node[ov.fake] (H1) at (2, 2) {$H_1$};
  \node[ov.fake] (H2) at (4, 2) {$H_2$};
  \node[ov.fake] (H3) at (6, 2) {$H_3$};

  \draw[arrow] (H0) to[out=0,in=90] (A0);
  \draw[arrow] (H1) to[out=0,in=90] (A1);
  \draw[arrow] (H2) to[out=0,in=90] (A2);
  \draw[arrow] (H3) to[out=0,in=90] (A3);

\end{tikzpicture}
  }
  \caption{The graphical model induced by the environment dynamics and agent
  policy.  RVs are shown as solid nodes, observed RVs in gray, and latent RVs
in white.  The history RVs, shown as dashed nodes, are aggregates of other RVs,
i.e., the previous actions and observations.}\label{fig:graphmodel}
\end{figure}

Some of the theory and results developed in this document concerns whether
certain RVs of interest are well-defined; therefore, we review the RVs defined
by POMDPs.
The environment dynamics and the agent policy jointly induce a graphical model
(see \Cref{fig:graphmodel}) over \emph{timed} RVs $S_t$, $A_t$, and $O_t$.
Note that only \emph{timed} RVs are defined directly, and there are no
intrinsically \emph{time-less} RVs.  Any other RV must be defined in terms of
the available ones, e.g. we can define a joint RV for \emph{timed} histories
$H_t \defeq (A_0, O_0, \ldots, A_{t-1}, O_{t-1})$.  Sometimes it is possible to
define a \emph{limiting} (stationary) state RV $S \defeq \lim_{t\to\infty}
S_t$, however it is never possible to define a limiting (stationary) history RV
$H$, since the sample space of each \emph{timed} RV $H_t$ is different, and
$\lim_{t\to\infty} H_t$ does not exist.  In essence, $H_t$ is inherently timed.

A probability is a numeric value associated with the assignment of a value $x$
from a sample space $\xset$ to an RV $X$, e.g., $\Pr(X=x)$.  Although it is
common to use simplified notation to informally omit the RV assignment (e.g.,
$\Pr(x)$), it must always be implicitly clear which RV ($X$) is involved in the
assignment.
In the reinforcement learning graphical model, a probability is well-defined if
and only if
\begin{enumerate*}[label=(\alph*)]
  \item it is grounded (implicitly or explicitly) to \emph{timed} RVs (or
    functions thereof); or
  \item it is time-invariant (i.e., it can be impicitly grounded to any time
    index).
\end{enumerate*}
For example, $\Pr(s'\mid s, a)$ is implicitly grounded to the RVs of a state
transition $\Pr(S_{t+1}=s'\mid S_t=s, A_t=a)$, and although the time-index $t$
is not clear from context, the probability is time-invariant and thus well
defined.
As another example, $\Pr(s\mid h)$ is implicitly grounded to the RVs of a
belief $\Pr(S_t=s\mid H_t=h)$, where the time-index $t$ is implicitly grounded
to the history length $t=|h|$, which makes the probability well defined.

\subsection{(Symmetric) Actor-Critic for POMDPs}\label{sec:a2c}

Policy gradient methods~\cite{sutton_policy_2000} for fully observable control
can be adapted to partial observable control by replacing occurrences of the
system state $s$ with the history $h$ (which is the Markov-state of an
equivalent \emph{history}-MDP).
In advantage actor-critic methods (A2C)~\cite{konda_actor-critic_2000}, a
policy model $\policy\colon \hset\to\Delta\aset$ parameterized by $\pparams$ is
trained using gradients estimated from sample data, while a critic model
$\vmodel\colon \hset\to\realset$ parameterized by $\cparams$ is trained to
predict history values $\vpolicy(h)$.  Note that we annotate parametric critic
models with a hat $\vmodel$, to distinguish them from their analytical
counterparts $\vpolicy$.  In A2C, the critic is used to bootstrap return
estimates and as a baseline, both of which are techniques for the reduction of
estimation variance~\cite{greensmith_variance_2004}.
The actor and critic models are respectively trained on $\ploss(\pparams) +
\lambda \hloss(\pparams)$ and $\closs(\cparams)$.

\paragraph{Policy Loss}
The \emph{policy loss} $\ploss(\pparams) \defeq -\Exp\left[ \sum_{t=0}^\infty
\gamma^t \rfn(s_t, a_t) \right]$ encodes the agent's performance as the
expected return.  The policy gradient
theorem~\cite{sutton_policy_2000,konda_actor-critic_2000} provides an
analytical expression for the policy loss gradient w.r.t.\ the policy
parameters,
\begin{equation}
  \nabla_\pparams \ploss(\pparams) = -\Exp\left[ \sum_{t=0}^\infty \gamma^t
  \qpolicy(h_t, a_t) \nabla_\pparams \log\policy(a_t; h_t) \right] \,.
  \label{eq:policy-gradient}
\end{equation}

Value $\qpolicy(h_t, a_t)$ is replaced by the \emph{temporal difference (TD)
error} $\delta_t$ to reduce variance (at the cost of introducing modeling
bias),
\begin{align}
  \nabla_\pparams \ploss(\pparams) &= -\Exp\left[ \sum_{t=0}^\infty \gamma^t
  \delta_t \nabla_\pparams \log\policy(a_t; h_t) \right] \,, \\
  \delta_t &\defeq \rfn(s_t, a_t) + \gamma \vmodel(h_{t+1}) - \vmodel(h_t) \,.
  \label{eq:tderror:h}
\end{align}

\paragraph{Critic Loss}
The \emph{critic loss} $\closs(\cparams) \defeq \Exp\left[ \sum_{t=0}^\infty
\delta_t^2 \right]$ is used to minimize the total TD error, the gradient of
which should propagate through $\vmodel(h_t)$, but not through the
bootstrapping $\vmodel(h_{t+1})$.

\paragraph{Negative-Entropy Loss}
Finally, the \emph{negative-entropy loss} is commonly used, $\hloss(\pparams)
\defeq - \Exp\left[ \sum_t \Ent\left[ \pi(A_t; h_t) \right] \right]$, in
combination with a decaying weight $\lambda$, to avoid premature convergence of
the policy model and to promote exploration~\cite{williams_function_1991}.

\subsection{Asymmetric Actor-Critic for POMDPs}\label{sec:bg:aa2c}

While asymmetric actor-critic can be understood to be an entire family of
methods which use critic asymmetry, for the remainder of this document we will
be specifically referring to a \emph{non-reactive} and \emph{non-deterministic}
variant of the work by Pinto et al.~\cite{pinto_asymmetric_2017}, which uses
critic asymmetry to address image-based robot learning.
Their work uses a reactive variant of \emph{deep deterministic policy gradient}
(DDPG)~\cite{lillicrap_continuous_2015} trained in simulation, and replaces
the reactive observation critic $\vmodel(o)$ with a state critic $\vmodel(s)$;
the variant we will be analyzing applies the same critic substitution to A2C.
In practice, this state-based asymmetry is obtained by replacing the TD error
of \Cref{eq:tderror:h} (used in both the policy and critic losses) with
\begin{equation}
  \delta_t = \rfn(s_t, a_t) + \gamma \vmodel(s_{t+1}) - \vmodel(s_t) \,.
  \label{eq:tderror:s}
\end{equation}

Although~\cite{pinto_asymmetric_2017} claim that their work
addresses partial observability, their evaluation is based on reactive
environments which are effectively fully observable;  while the agent only
receives a single image, each image provides a virtually \emph{complete} and
\emph{occlusion-free} view of the entire workspace.  In practice, the images
are merely high-dimensional representations of a compact state.

\section{Theory of Asymmetric Actor-Critic}\label{sec:aa2c}

In this section, we analyze the theoretical implications of using a state
critic under partial observability, as described in \Cref{sec:bg:aa2c}, and
expose critical underlying issues.
The primary result will be that the time-invariant state value function
$\vpolicy(s)$ of a non-reactive agent is generally ill-defined.
Then, we show that the time-invariant state value function $\vpolicy(s)$ of a
reactive agent is well-defined under mild assumptions, but generally introduces
a bias into the training process which may undermine learning.
Finally, we show that the time-invariant state value function $\vpolicy(s)$ of
a reactive agent under stronger assumptions can be both well-defined and
unbiased.
Later, in \Cref{sec:uaa2c}, we provide a more general alternative which
guarantees well-defined and unbiased time-invariant state-based value functions
for arbitrary policies and control problems.

Informally, the issue with $\vpolicy(s)$ is that the state alone does not
contain sufficient information to determine the agent's future behavior---which
generally depends on the history---and is thus unable to accurately represent
expected future returns.  Ironically, state values suffer from a form of
\emph{history aliasing}, i.e., being unable to infer the agent's history from
the system's state.  This is particularly evident in control problems which
require the agent to perform forms of information gathering (a common
occurrence in partially observable control) which are not reflected in the
system state, e.g., reach a certain spot to observe a piece of information
which is necessary to determine future optimal behavior and solve the control
task.  In such cases, the state alone does not generally indicate whether the
agent has collected the necessary information in the past or not, and is
therefore unable to represent adequately whether the current state is a
positive or negative occurrence.  Formally, we will show that $\vpolicy(s)$ is
generally not a well-defined quantity and, even in special cases where it is
well-defined, generally introduces a bias in the learning process caused by the
imperfect correlation between histories and states;  in essence, the average
value of histories inferred from the current state is not an accurate estimate
of the current history's value.

\paragraph{Methodology}
We note that replacing the history critic is intrinsically questionable: the
policy gradient theorem for POMDPs (\Cref{eq:policy-gradient}) specifically
requires history values, and replacing them with other state-based values will
generally result in biased gradients and a general loss of theoretical
guarantees.  Therefore, we analyze state values $\vpolicy(s)$ as stochastic
estimators of history values $\vpolicy(h)$ and consider the corresponding
estimation bias, i.e., the difference between the expected estimate
$\Exp_{s\mid h}\left[ \vpolicy(s) \right]$ and the ground truth estimation
target $\vpolicy(h)$ for any given history $h$.

\subsection{General Policy under Partial Observability}\label{sec:general}

A policy's state value function $\vpolicy\colon\sset\to\realset$ is
\emph{tentatively} defined as the expected return following a realizable state
$s$,
\begin{equation}
  \vpolicy(s\iter{0}) = \Exp_{\bar s, \bar a \mid s\iter{0}}\left[
  \sum_{k=0}^\infty \gamma^k R(s\iter{k}, a\iter{k}) \right] \,,
  \label{eq:vs-raw}
\end{equation}
\noindent which, if well-defined, supports an indirect recursive Bellman form,
\begin{align}
  \vpolicy(s) &= \sum_{a\in\aset} \Pr(a\mid s) \qpolicy(s, a) \,, \label{eq:vs}
  \\
  \qpolicy(s, a) &= R(s, a) + \gamma \Exp_{s'\mid s, a}\left[ \vpolicy(s')
  \right] \,. \label{eq:qsa}
\end{align}

In \Cref{eq:vs}, we note the term $\Pr(a\mid s)$, which encodes the likelihood
of an action being taken from a given state.  Because the agent policy depends
on histories (not states), this term is not directly available, but must be
derived indirectly by integrating over possible histories.  Further, because
$s$ is timeless, and no additional context is available to narrow down time,
there is no choice but to integrate over histories of all possible lengths.
\begin{equation}
  \Pr(a\mid s) = \sum_{h\in\hset} \Pr(h\mid s) \pi(a; h) \,. \label{eq:pras}
\end{equation}

\Cref{eq:pras} reveals the probability term $\Pr(h\mid s)$, which encodes the
likelihood of a history having taken place in the past given a current state.
While $\Pr(h\mid s)$ may look harmless, it is the underlying cause of serious
analytical issues.  As discussed in \Cref{sec:graphmodel}, a probability is
only well-defined if associated with well-defined RVs, and unfortunately such
RVs do not exist for $\Pr(h\mid s)$.
On one hand, timed RVs $\Pr(H_t=h\mid S_t=s)$ cannot be used, because
\Cref{eq:pras} integrates over the sample space of all histories, and not just
those of a given length $t$.
On the other hand, time-less RVs $\Pr(H=h\mid S=s)$ cannot be used, because
such time-less RVs do not exist in the RL graphical model.  
Ultimately, $\Pr(h\mid s)$ is mathematically ill-defined, which consequently
causes both $\Pr(a\mid s)$ and $\vpolicy(s)$ to be ill-defined as well.

\begin{theorem}\label{thm:vs}
  In partially observable control problems, a time-invariant state value
  function $\vpolicy(s)$ is generally ill-defined.
\end{theorem}

The practical implications of an ill-defined value function are not obvious;
even though the analytical value function $\vpolicy(s)$ is ill-defined, the
state critic's $\vmodel(s)$ training process is based on valid calculations
over sample data, which results in syntactically valid updates of the critic
parameters.  However, given that asymptotic convergence is theoretically
impossible when $\vpolicy(s)$ is ill-defined, the critic's target will continue
shifting indefinitely based on the recent batches of training data, even when
unbiased Monte Carlo return estimates are used to train the critic (without
bootstrapping).
In practice, the effects are not necessarily catastrophic for all control
problems, and likely vary depending on the amount of partial observability, on
the agent's need to gather and remember information, and on the specific state
and observation representations.

In principle, \emph{timed} value functions $\vpolicy_t(s)$ represent a
straightforward solution to all these issues (see
appendix~\cite{baisero_unbiased_2022}).  However, learning a timed critic model
is likely to pose additional learning challenges, due to the need to generalize
well and accurately across time-steps.
Rather, we will demonstrate that there are special cases of the general control
problem which do guarantee well-defined time-invariant value functions
$\vpolicy(s)$ (see \Cref{sec:special:reactive:po,sec:special:reactive:fo}).
However, before that, we can already show that, even when $\vpolicy(s)$ is
guaranteed to be well-defined, it is not guaranteed to be unbiased.

\begin{theorem}\label{thm:vs:bias}
  Even when well-defined, a time-invariant state value function $\vpolicy(s)$
  is generally a biased estimate of $\vpolicy(h)$, i.e., it is not guaranteed
  that $\vpolicy(h) = \Exp_{s\mid h}\left[ \vpolicy(s) \right]$.
\end{theorem}

\begin{proof}
  Consider two histories which are different, $h'\neq h''$, and result in
  different action distributions, $\policy(A; h') \neq \policy(A; h'')$, but
  are associated with the same belief, $b(h') = b(h'')$---a fairly common
  occurrence in many POMDPs (see appendix~\cite{baisero_unbiased_2022}).  On
  one hand, because the two histories result in different behaviors, future
  trajectories and rewards will differ, leading to different history values,
  $\vpolicy(h') \neq \vpolicy(h'')$.
  On the other hand, because the two beliefs are equal, the expected state
  values must also be equal, $\Exp_{s\mid h'}\left[ \vpolicy(s) \right] =
  \Exp_{s\mid h''}\left[ \vpolicy(s) \right]$.
  If equation $\vpolicy(h) = \Exp_{s\mid h}\left[ \vpolicy(s) \right]$ held for
  all histories, then it would hold for $h'$ and $h''$ too, which implies
  $\vpolicy(h') = \Exp_{s\mid h'}\left[ \vpolicy(s) \right] = \Exp_{s\mid
  h''}\left[ \vpolicy(s) \right] = \vpolicy(h'')$ ---a simple contradiction.
  Therefore, either $\vpolicy(h') \neq \Exp_{s\mid h'}\left[ \vpolicy(s)
  \right]$ or $\vpolicy(h'') \neq \Exp_{s\mid h''}\left[ \vpolicy(s) \right]$
  (or both).
\end{proof}

\subsection{Reactive Policy under Partial
Observability}\label{sec:special:reactive:po}

We show that $\vpolicy(s)$ is well-defined if we make two assumptions about the
agent and environment:
\begin{enumerate*}[label=(\alph*)]
  \item that the policy is reactive (a common but inadequate assumption); and
  \item that the POMDP observation function depends only on the current state,
    $\ofn\colon\sset\to\Delta\oset$, rather than the entire state transition (a
    mild assumption).
\end{enumerate*}
Under these assumptions, we can expand $\Pr(a\mid s)$ by integrating over the
space of all observations (rather than all histories),
\begin{equation}
  \Pr(a\mid s) = \sum_{o\in\oset} \Pr(o\mid s) \pi(a; o) \,.
\end{equation}
\noindent In this case, $\Pr(o\mid s)$ is time-invariant, and can therefore be
implicitly grounded to RVs of any time index $\Pr(O_t=o \mid S_t=s)$.  This
leads to a well-defined value $\vpolicy(s)$ which, however, generally remains
\emph{biased} compared to $\vpolicy(h)$, per \Cref{thm:vs:bias}.
In addition to \Cref{thm:vs:bias}, which is applicable in a more general
setting, see appendix~\cite{baisero_unbiased_2022} for two additional proofs
which also take into account the specific assumptions made here.
Broadly speaking, the bias is caused by the fact that hidden in $\vpolicy(s)$
is an expectation over observations $o$ which are not necessarily consistent
with the true history $h$;  each proof covers this issue from different angles.

Although the value function is well-defined under reactive control, there are
still two significant issues which preclude these assumptions from representing
a general solution:
\begin{enumerate*}[label=(\alph*)]
  \item reactive policies are inadequate to solve many POMDPs; and
  \item the value function bias may prevent the agent from learning a
    satisfactory behavior.
\end{enumerate*}

\subsection{Reactive Policy under Full
Observability}\label{sec:special:reactive:fo}

We show that the state value function is both well-defined and unbiased under
two assumptions:
\begin{enumerate*}[label=(\alph*)]
  \item that the policy is reactive (a common but inadequate assumption); and
  \item that there is a bijective abstraction $\phi\colon\oset\to\sset$ between
    observations and states (an unrealistic assumption).
\end{enumerate*}
The abstraction $\phi$ encodes the fact that the environment is not truly
partially observable, but rather that states and observations fundamentally
contain the same information, albeit at different levels of abstraction.  For
example, in the control problems used by Pinto et
al.~\cite{pinto_asymmetric_2017}, and an image displaying a workspace without
occlusions is a low-level abstraction (observation), while a concise vector
representation of the object poses in the workspace are a high-level
abstraction (state).

In this case, the action probability term $\Pr(a\mid s)$ does not need to be
obtained indirectly by integrating other variables;  rather, bijection $\phi$
can be used to relate it to the policy model $\Pr(a\mid s) = \pi(a;
\phi\I(s))$.  Contrary to the previous cases, the overall state value function
$\vpolicy(s)$ is not only well-defined, but also unbiased.

\begin{theorem}\label{thm:vs.reactive.fo}
  If the POMDP states and observations are related by a bijection
  $\phi\colon\oset\to\sset$, and the policy is reactive, then $\vpolicy(s)$ is
  an unbiased estimate of $\vpolicy(h)$, i.e., $\vpolicy(h) = \Exp_{s\mid
  h}\left[ \vpolicy(s) \right]$.
\end{theorem}
\begin{proof}
  The bijection between $o_h$ and $s$ not only implies a many-to-one
  relationship between histories and states, but also fully determines the
  agent's state-conditioned action.  In the following derivation, we use these
  facts to determine the first action and reward, a process which can be
  repeated indefinitely for future actions and rewards.
  \begin{align}
    \Exp_{s\mid h}\left[ \vpolicy(s) \right] &= \Exp_{s\mid h}\left[
    \sum_{a\in\aset} \Pr(a\mid s) \qpolicy(s, a) \right] \nonumber \\
    &= \Exp_{s\mid h}\left[ \sum_{a\in\aset} \pi(a; o_h) \qpolicy(s, a) \right]
    \nonumber \\
    &= \sum_{a\in\aset} \pi(a; o_h) \Exp_{s\mid h}\left[ \qpolicy(s, a) \right]
    \nonumber \\
    &= \sum_{a\in\aset} \pi(a; o_h) \Exp_{s\mid h}\left[ R(s, a) +
    \gamma\Exp_{s'\mid s, a}\left[ \vpolicy(s') \right] \right] \nonumber \\
    &= \sum_{a\in\aset} \pi(a; o_h) \left( R(h, a) + \gamma\Exp_{s'\mid h,
    a}\left[ \vpolicy(s') \right] \right) \nonumber \\
    &= \sum_{a\in\aset} \pi(a; o_h) \left( R(h, a) + \gamma\Exp_{o\mid h,
    a}\left[ \Exp_{s'\mid hao}\left[ \vpolicy(s') \right] \right] \right) \nonumber \\
    \intertext{(repeat process until end of episode)}
    &= \sum_{a\in\aset} \pi(a; o_h) \left( R(h, a) + \gamma\Exp_{o\mid h,
    a}\left[ \vpolicy(hao) \right] \right) \nonumber \\
    &= \sum_{a\in\aset} \pi(a; o_h) \qpolicy(h, a) \nonumber \\
    &= \vpolicy(h) \,.
  \end{align}

\end{proof}

The benefit of using a state critic under this scenario is that the critic
model can avoid learning a representation of the observations before learning
the values~\cite{pinto_asymmetric_2017}.
Naturally, the main disadvantage of this scenario is that most POMDPs do not
satisfy the bijective abstraction assumption;  if anything, this assumption is
intrinsically incompatible with partial observability, and any POMDP which
satisfies this assumption is really an MDP in disguise.  Nonetheless, if a
control problem only deviates mildly from full observability, it is likely that
a state critic will benefit the learning agent despite the theoretical issues.

\section{Unbiased Asymmetric Actor-Critic}\label{sec:uaa2c}

In this section, we introduce \emph{unbiased asymmetric actor-critic}, an
actor-critic variant able to exploit asymmetric state information during
offline training while avoiding the issues of state value functions exposed in
\Cref{sec:aa2c}.
Consider the \emph{history-state} value function $\vpolicy(h,
s)$~\cite{bono_study_2018}, defined as the expected return following a
realizable history-state pair $h$ and $s$,
\begin{equation}
  \vpolicy(h\iter{0}, s\iter{0}) = \Exp_{\bar s, \bar a \mid
  h\iter{0},s\iter{0}}\left[ \sum_{k=0}^\infty \gamma^k R(s\iter{k}, a\iter{k})
\right] \,, \label{eq:vhs-raw}
\end{equation}
\noindent which supports an indirect recursive Bellman form,
\begin{align}
  \vpolicy(h, s) &= \sum_{a\in\aset} \policy(a; h) \qpolicy(h, s, a) \,,
  \label{eq:vhs} \\
  \qpolicy(h, s, a) &= \rfn(s, a) + \gamma \Exp_{s',o\mid s,a}\left[
  \vpolicy(hao, s') \right] \,. \label{eq:qhsa}
\end{align}

Note that the history $h$ and state $s$ cover different and orthogonal roles:
the history $h$ determines the future behavior of the agent, while the state
$s$ determines the future behavior of the environment.  Compared to the history
value $\vpolicy(h)$, the state information in $\vpolicy(h, s)$ provides
additional context to determine the agent's true underlying situation, its
rewards, and its expected return.  Compared to the state value $\vpolicy(s)$,
the history information in $\vpolicy(h, s)$ provides additional context to
determine the agent's future behavior, which guarantees that $\vpolicy(h, s)$
is well-defined and unbiased.

\begin{theorem}\label{thm:vhs}
  For arbitrary control problems and policies, $\vpolicy(h, s)$ is an unbiased
  estimate of $\vpolicy(h)$, i.e., $\vpolicy(h) = \Exp_{s\mid h}\left[
  \vpolicy(h, s) \right]$.
\end{theorem}

\begin{proof}
  Follows from \Cref{eq:vh-raw,eq:vhs-raw},
  \begin{align}
    \vpolicy(h\iter{0}) &= \Exp_{\bar s,\bar a\mid h\iter{0}}\left[
    \sum_k \gamma^k R(s\iter{k}, a\iter{k}) \right] \nonumber \\
    &= \Exp_{s\iter{0}\mid h\iter{0}} \Exp_{\bar s,\bar a\mid
    h\iter{0},s\iter{0}}\left[ \sum_k \gamma^k R(s\iter{k}, a\iter{k}) \right]
    \nonumber \\
    &= \Exp_{s\iter{0}\mid h\iter{0}}\left[ \vpolicy(h\iter{0}, s\iter{0}) \right] \,.
  \end{align}
\end{proof}

As we have done for state values $\vpolicy(s)$, we are interested in the
properties of history-state values $\vpolicy(h, s)$ in relation to history
values $\vpolicy(h)$.
\Cref{thm:vhs} shows that history and history-state values are related by
$\vpolicy(h) = \Exp_{s\mid h}\left[ \vpolicy(h, s) \right]$, i.e., history-state
values are interpretable as \emph{Monte Carlo (MC) estimates} of the respective
history values.  In expectation, history-state values provide the same
information as the history values, therefore an asymmetric variant of the
policy gradient theorem can be formulated.

\begin{theorem}[Asymmetric Policy
  Gradient]\label{thm:asymmetric-policy-gradient}
  \begin{equation}
    \nabla_\pparams\ploss(\pparams) = -\Exp\left[ \sum_t \gamma^t \qpolicy(h_t,
    s_t, a_t) \nabla_\pparams \log\policy(a_t; h_t) \right] \,.
  \end{equation}
\end{theorem}

\begin{proof}
  Following \Cref{thm:vhs}, we have
  \begin{align}
    \qpolicy(h, a) &= R(h, a) + \gamma\Exp_{o\mid h, a}\left[ \vpolicy(hao)
    \right] \nonumber \\
    &= R(h, a) + \gamma\Exp_{o\mid h, a}\left[ \Exp_{s'\mid h, a, o}\left[
    \vpolicy(hao, s') \right] \right] \nonumber \\
    &= R(h, a) + \gamma\Exp_{s', o\mid h, a}\left[ \vpolicy(hao, s') \right]
    \nonumber \\
    &= \Exp_{s\mid h}\left[ R(s, a) + \gamma\Exp_{s', o\mid s, a}\left[
    \vpolicy(hao, s') \right] \right] \nonumber \\
    &= \Exp_{s\mid h}\left[ \qpolicy(h, s, a) \right] \,.
  \end{align}
  Therefore,
  \begin{align}
    &\phantom{{}={}} \nabla_\pparams\ploss(\pparams) \nonumber \\
    &= -\Exp\left[ \sum_t \gamma^t \qpolicy(h_t,
    a_t) \nabla_\pparams \log\policy(a_t; h_t) \right] \nonumber \\
    &= -\sum_t \gamma_t \Exp_{h_t,a_t}\left[ \qpolicy(h_t, a_t) \nabla_\pparams
    \log\policy(a_t; h_t) \right] \nonumber \\
     &= -\sum_t \gamma^t \Exp_{h_t,a_t}\left[ \Exp_{s_t\mid h_t}\left[
     \qpolicy(h_t, s_t, a_t) \right] \nabla_\pparams \log\policy(a_t; h_t)
   \right] \nonumber \\
    &= -\sum_t \gamma^t \Exp_{h_t,s_t,a_t}\left[ \qpolicy(h_t, s_t, a_t)
    \nabla_\pparams \log\policy(a_t; h_t) \right] \nonumber \\
    &= -\Exp\left[ \sum_t \gamma^t \qpolicy(h_t, s_t, a_t) \nabla_\pparams
    \log\policy(a_t; h_t) \right] \,.
  \end{align}
\end{proof}

As estimators, history-state values $\vpolicy(h, s)$ can be described in terms
of their bias and variance w.r.t.\ history values $\vpolicy(h)$.  Beyond
providing the inspiration for the MC interpretation, \Cref{thm:vhs} already
proves that $\vpolicy(h, s)$ is unbiased, while its variance is dynamic and
depends on the history $h$ via the belief-state $\Pr(S\mid h)$;  in particular,
low-uncertainty belief-states result in low variance, and deterministic
belief-states result in no variance.
Given that operating optimally in a partially observable environment generally
involves information-gathering strategies associated with low-uncertainty
belief-states, the practical variance of the history-state value is likely to
be relatively low once the agent has learned to solve the task to some degree
of success.

Inspired by \Cref{thm:asymmetric-policy-gradient}, we propose \emph{unbiased
asymmetric A2C}, which uses a history-state critic
$\vmodel\colon\hset\times\sset\to\realset$ trained to model history-state
values $\vpolicy(h, s)$,
\begin{align}
  &\nabla_\pparams\ploss(\pparams) = -\Exp\left[ \sum_t \gamma^t \delta_t
  \nabla_\pparams \log\policy(a_t; h_t) \right] \,, \\
  &\delta_t = R(s_t, a_t) + \gamma\vmodel(h_{t+1}, s_{t+1}) - \vmodel(h_t, s_t)
  \,. \label{eq:tderror:hs}
\end{align}
Because $\vmodel(h, s)$ receives the history $h$ as input, it can still predict
reasonable estimates of the agent's expected future discounted returns;  and
because it receives the state $s$ as input, it is still able to exploit state
information while introducing no bias into the learning process, e.g., for the
purposes of bootstrapping the learning of critic values and/or aiding the
learning of history representations.

\subsection{Interpretations of State}

Although the history-state value is analytically well-defined, it remains
worthwhile to question why the inclusion of the state information should help
the actor-critic agent at all.  We attempt to address this open question, and
consider two competing interpretations, which we call
\emph{state-as-information} and \emph{state-as-a-feature}.

\paragraph{State as Information}
Under this interpretation, state information is valuable because it is latent
information unavailable in the history, which results in more informative
values which help train the policy.  However, we argue that this interpretation
is flawed for two reasons:
\begin{enumerate*}[label=(\alph*)]
  \item The policy gradient theorem specifically requires $\vpolicy(h)$, which
    contains precisely the correct information required to accurately estimate
    policy gradients.  In this context, history values already contain the
    correct type and amount of information necessary to train the policy, and
    there is no such thing as ``more informative values'' than history values.
  \item In theory, the history-state value in
    \Cref{thm:asymmetric-policy-gradient} could use any other state sampled
    according to $\tilde s\sim b(h)$, rather than the true system state, which
    would also result in the same analytical bias and variance properties.  In
    practice, we only use the true system state due to it being directly
    available during offline training;  however, we believe that its identity
    as the true system state is analytically irrelevant, which leads to the
    next interpretation of state.
\end{enumerate*}

\paragraph{State as a Feature}\label{sec:uaa2c:SAF}
We conjecture an alternative interpretation according to which the state can be
seen as a \emph{stochastic} high-level feature of the history.
Consider a history critic $\vmodel(h)$;  to appropriately model the value
function $\vpolicy(h)$, $\vmodel(h)$ must first learn an adequate history
representation, which is in and of itself a significant learning challenge.
The critic model would likely benefit from receiving auxiliary high-level
features of the history $\phi(h)$.  The resulting critic $\vmodel(h, \phi(h))$
remains fundamentally a history critic, as the auxiliary features are
exclusively a modeling/architecture construct.
Next, we consider what kind of high-level features $\phi(h)$ would be useful
for control.  While the specifics of what makes a good history representation
depend strongly on the task, there is a natural choice which is arguably useful
in many cases:  the belief-state $b(h)$.  Because the belief-state is a
sufficient statistic of the history for control, providing it to the critic
model $\vmodel(h, b(h))$ is likely to greatly improve its ability to generalize
across histories.
Finally, we conjecture that \emph{any} state sampled according to the
belief-state $s\sim b(h)$---including the true system state---can be considered
a \emph{stochastic} realization of the belief-state feature, resulting in the
history-state critic $\vmodel(h, s)$.
According to this interpretation, the importance of the state in the
history-state critic is not in its identity as the true system state, but as a
stochastic realization of hypothetical belief-state features, and presumably
any other state sampled from the belief-state $\tilde s\sim b(h)$ could be
equivalently used.

\section{Evaluation}\label{sec:evaluation}

\begin{figure*}[t!]
  \centering
  \begin{subfigure}{.66\linewidth}
    \centering
    \includegraphics[width=\linewidth]{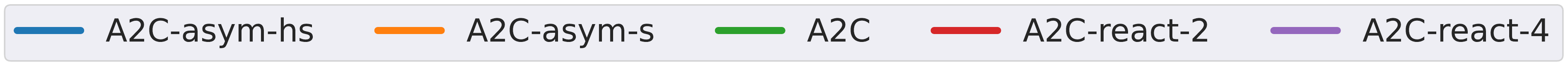}
  \end{subfigure}

  \begin{subfigure}{.24\linewidth}
    \centering
    \includegraphics[width=\linewidth]{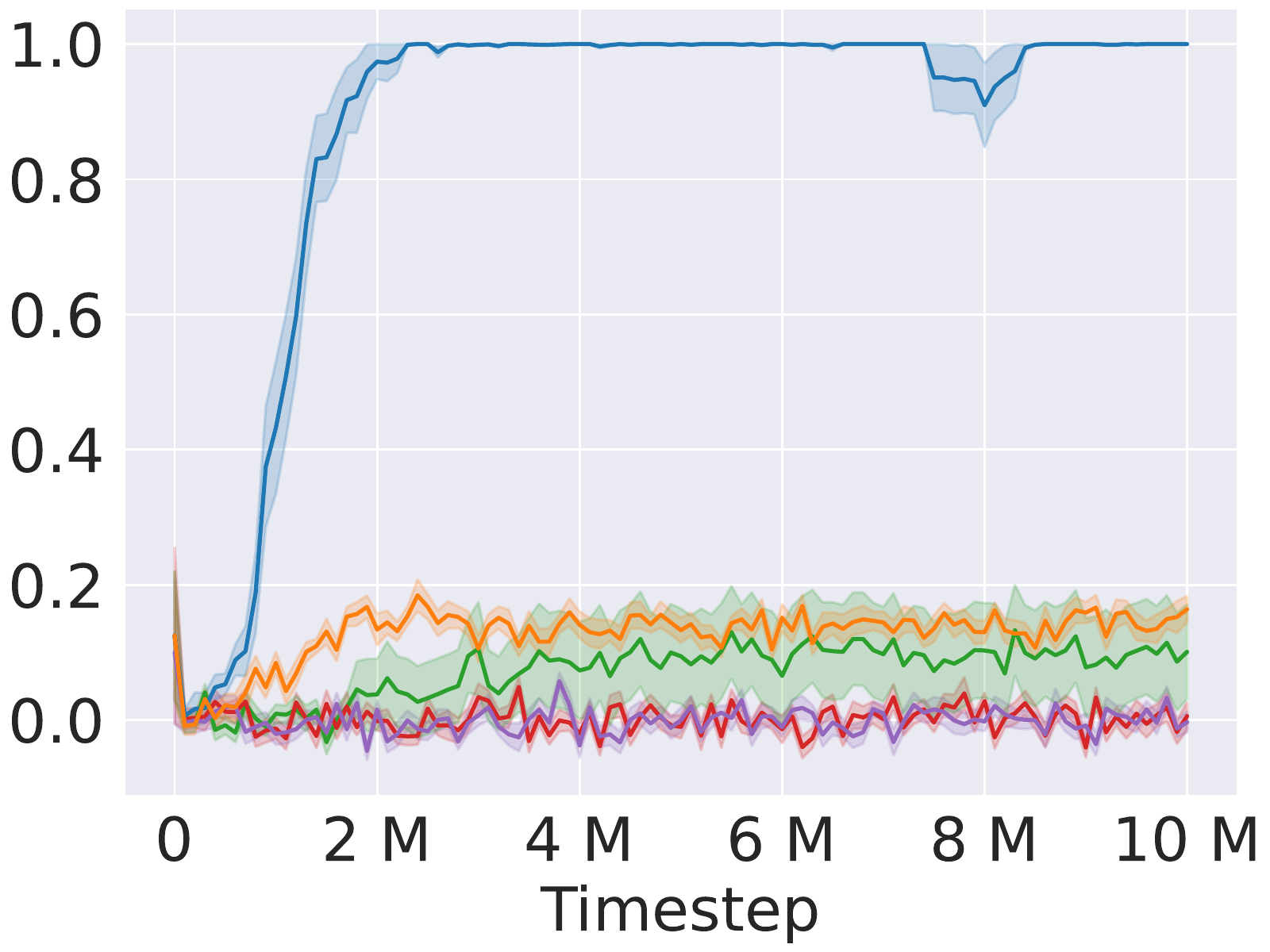}
    \caption{\heavenhellthree}\label{fig:performance:heavenhellthree}
  \end{subfigure}
  \begin{subfigure}{.24\linewidth}
    \centering
    \includegraphics[width=\linewidth]{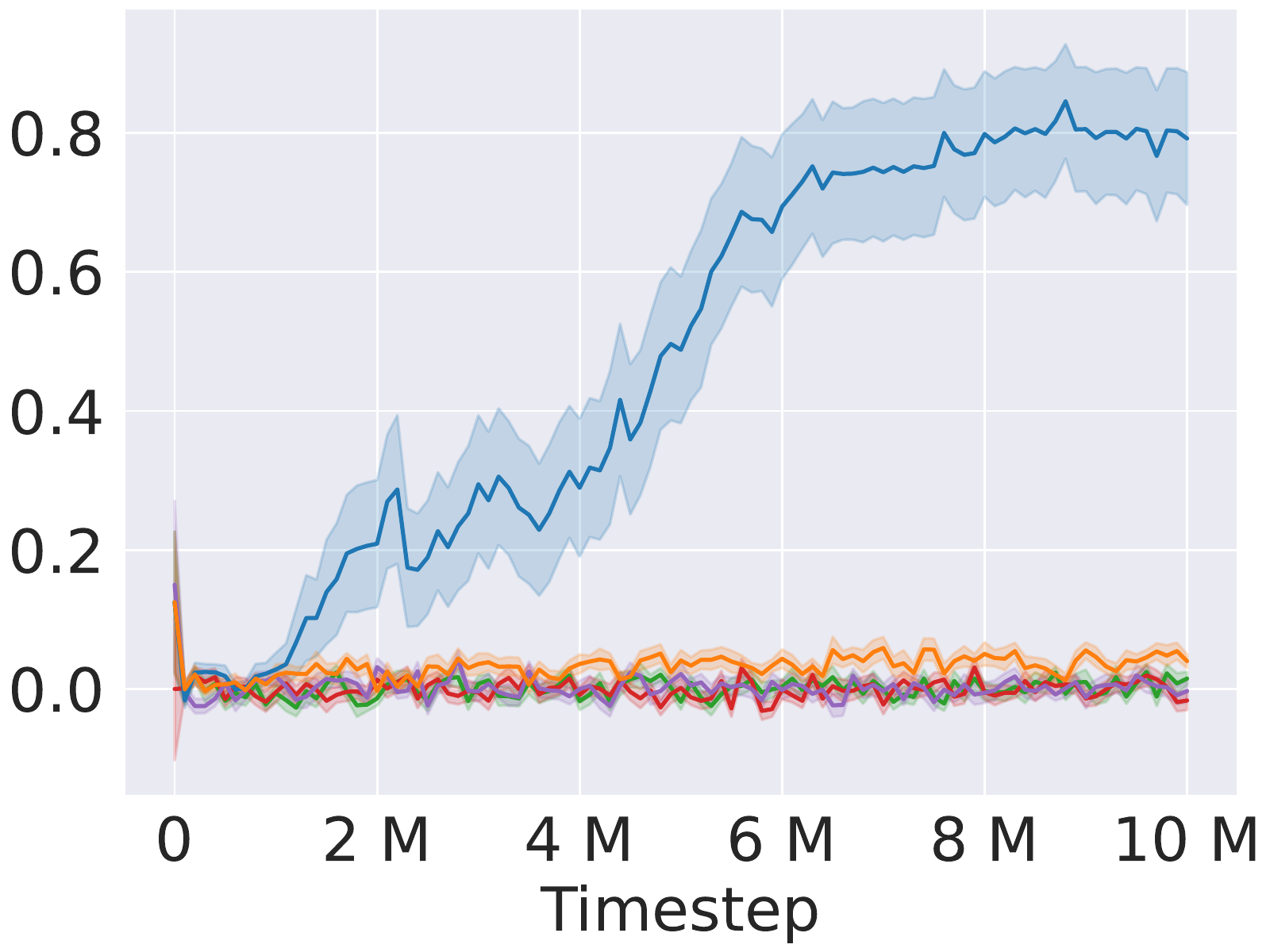}
    \caption{\heavenhellfour}\label{fig:performance:heavenhellfour}
  \end{subfigure}
  \begin{subfigure}{.24\linewidth}
    \centering
    \includegraphics[width=\linewidth]{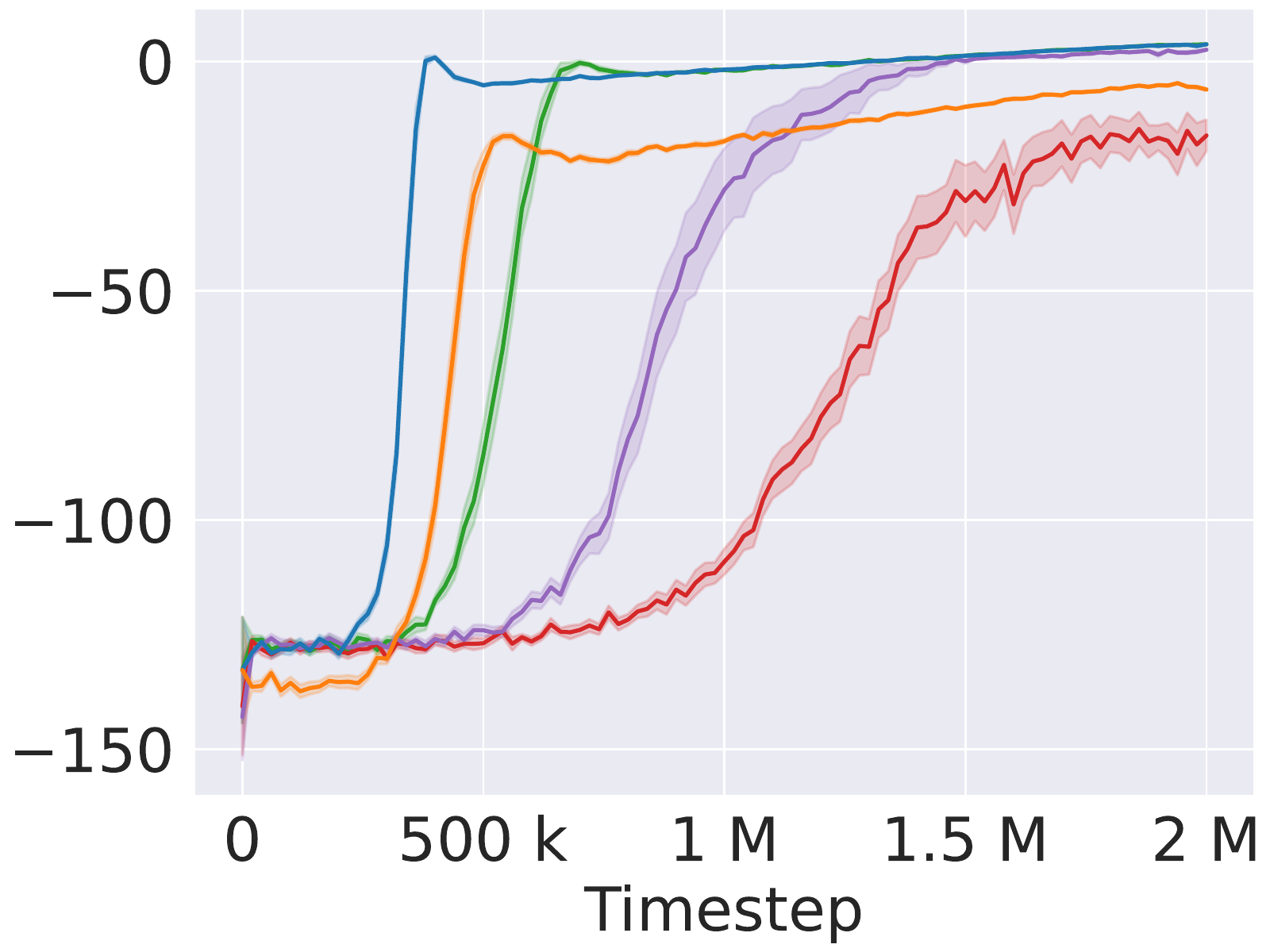}
    \caption{\shoppingfive}\label{fig:performance:shoppingfive}
  \end{subfigure}
  \begin{subfigure}{.24\linewidth}
    \centering
    \includegraphics[width=\linewidth]{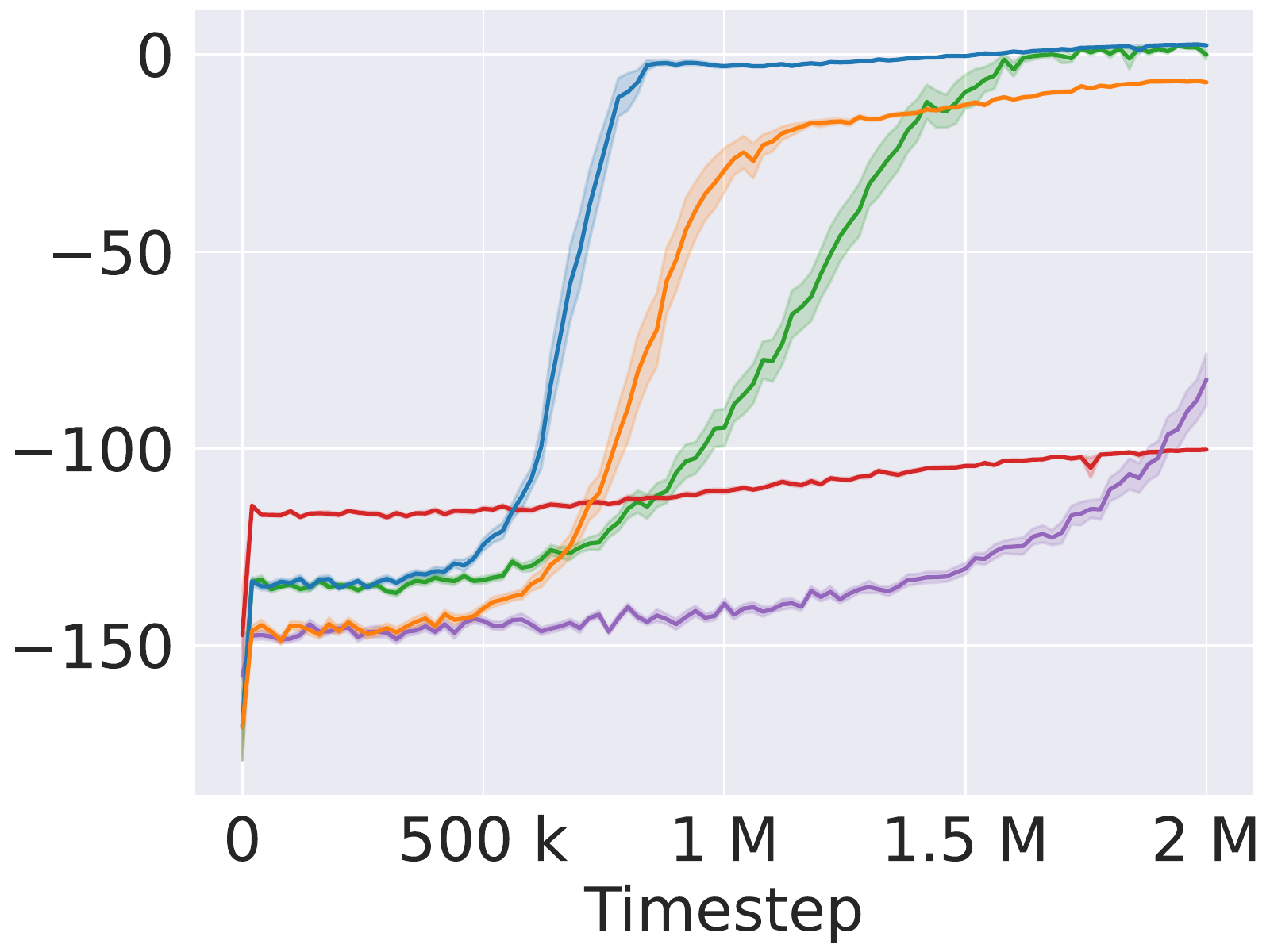}
    \caption{\shoppingsix}\label{fig:performance:shoppingsix}
  \end{subfigure}

  \begin{subfigure}{.24\linewidth}
    \centering
    \includegraphics[width=\linewidth]{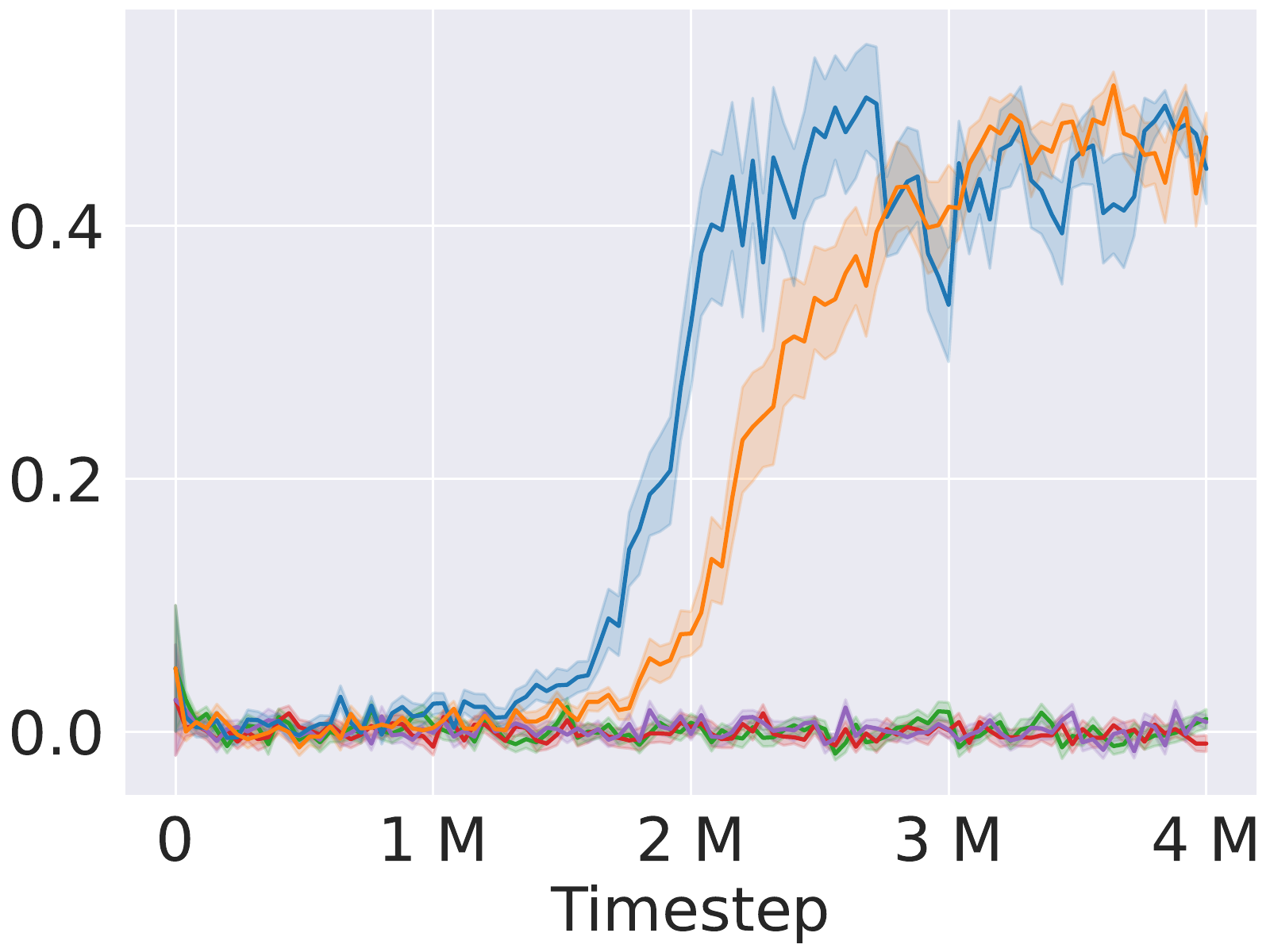}
    \caption{\carflag}\label{fig:performance:carflag}
  \end{subfigure}
  \begin{subfigure}{.24\linewidth}
    \centering
    \includegraphics[width=\linewidth]{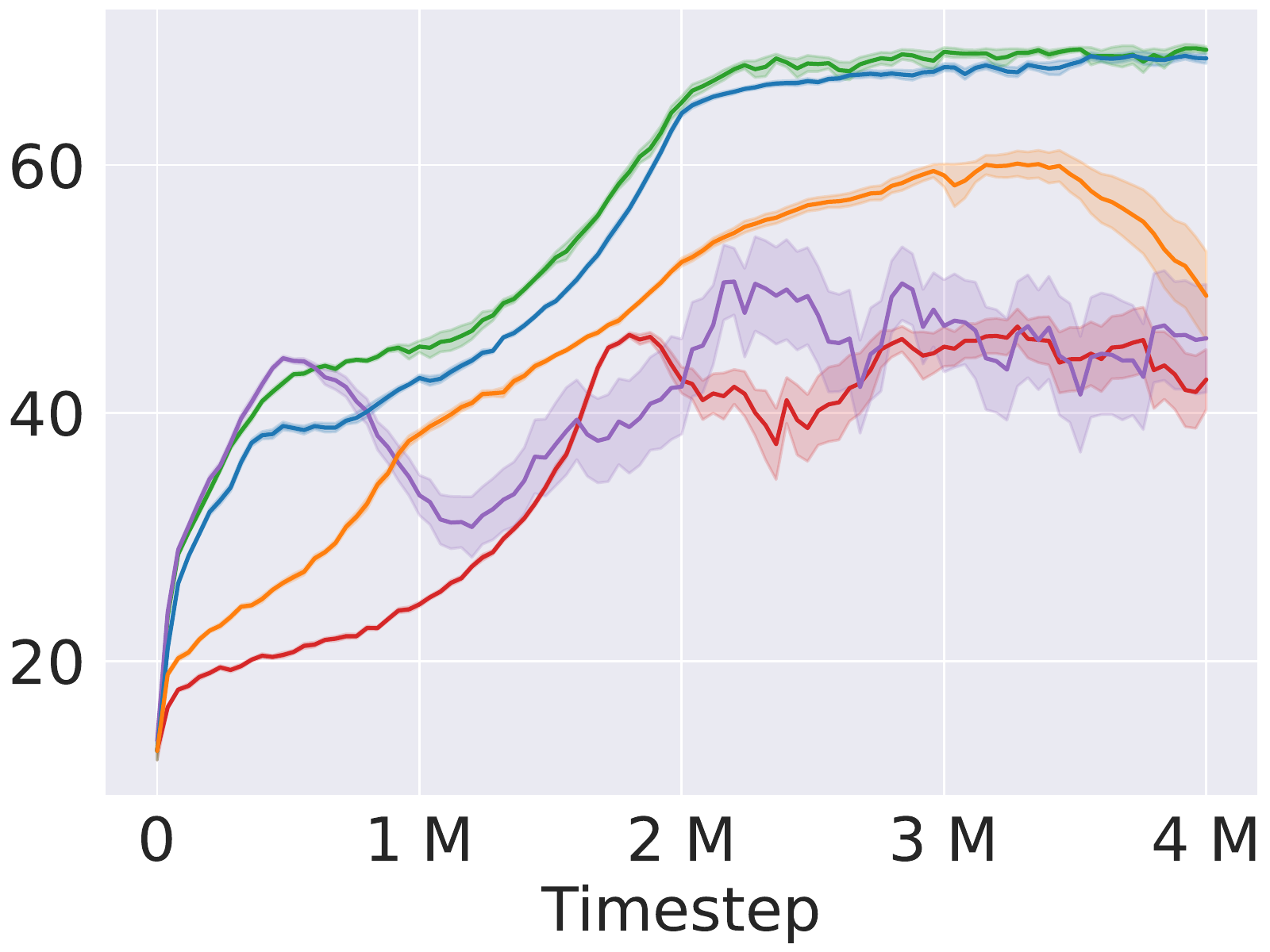}
    \caption{\cleaner}\label{fig:performance:cleaner}
  \end{subfigure}
  \begin{subfigure}{.24\linewidth}
    \centering
    \includegraphics[width=\linewidth]{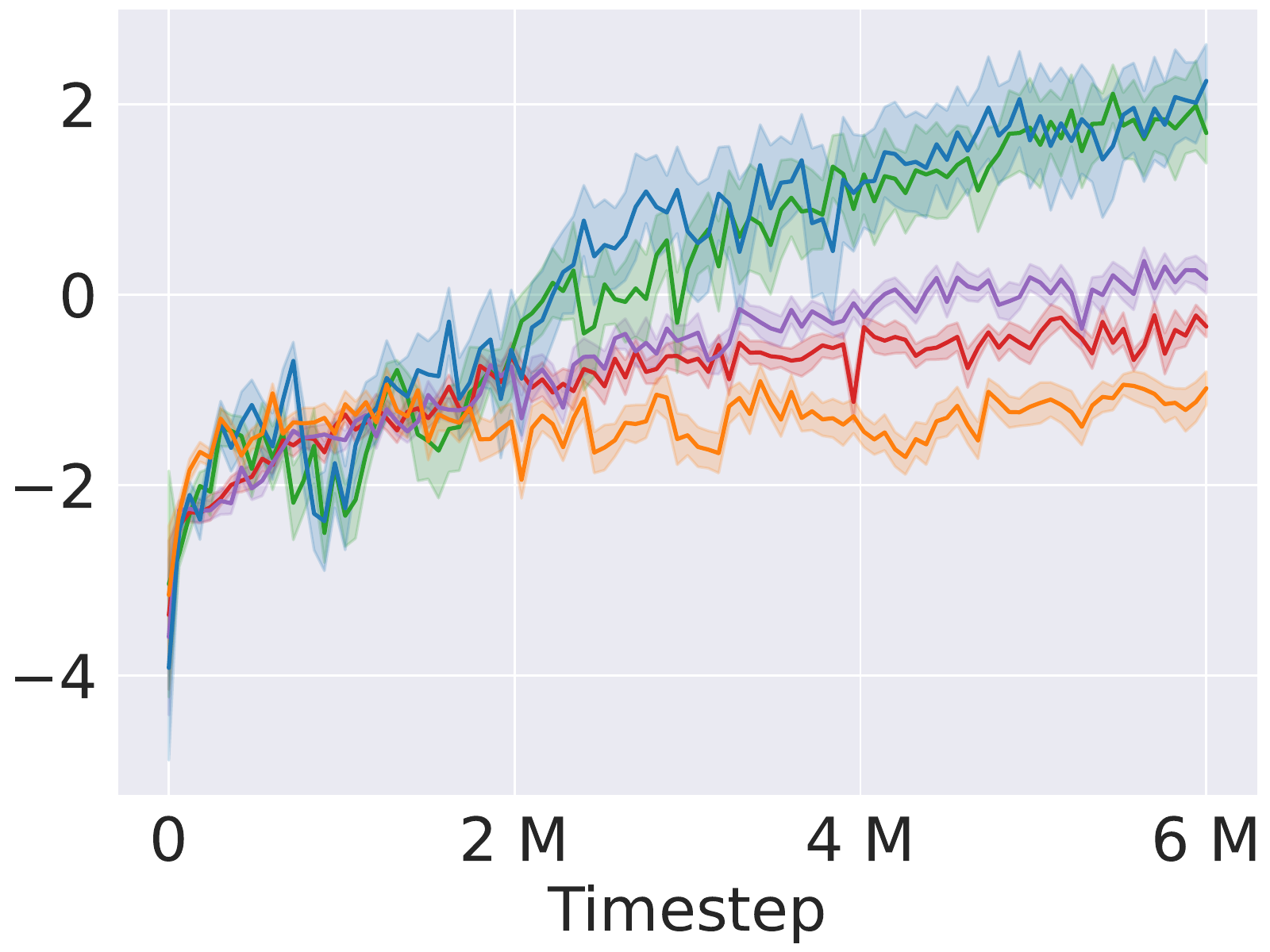}
    \caption{\memoryroomsseven}\label{fig:performance:memoryroomsseven}
  \end{subfigure}
  \begin{subfigure}{.24\linewidth}
    \centering
    \includegraphics[width=\linewidth]{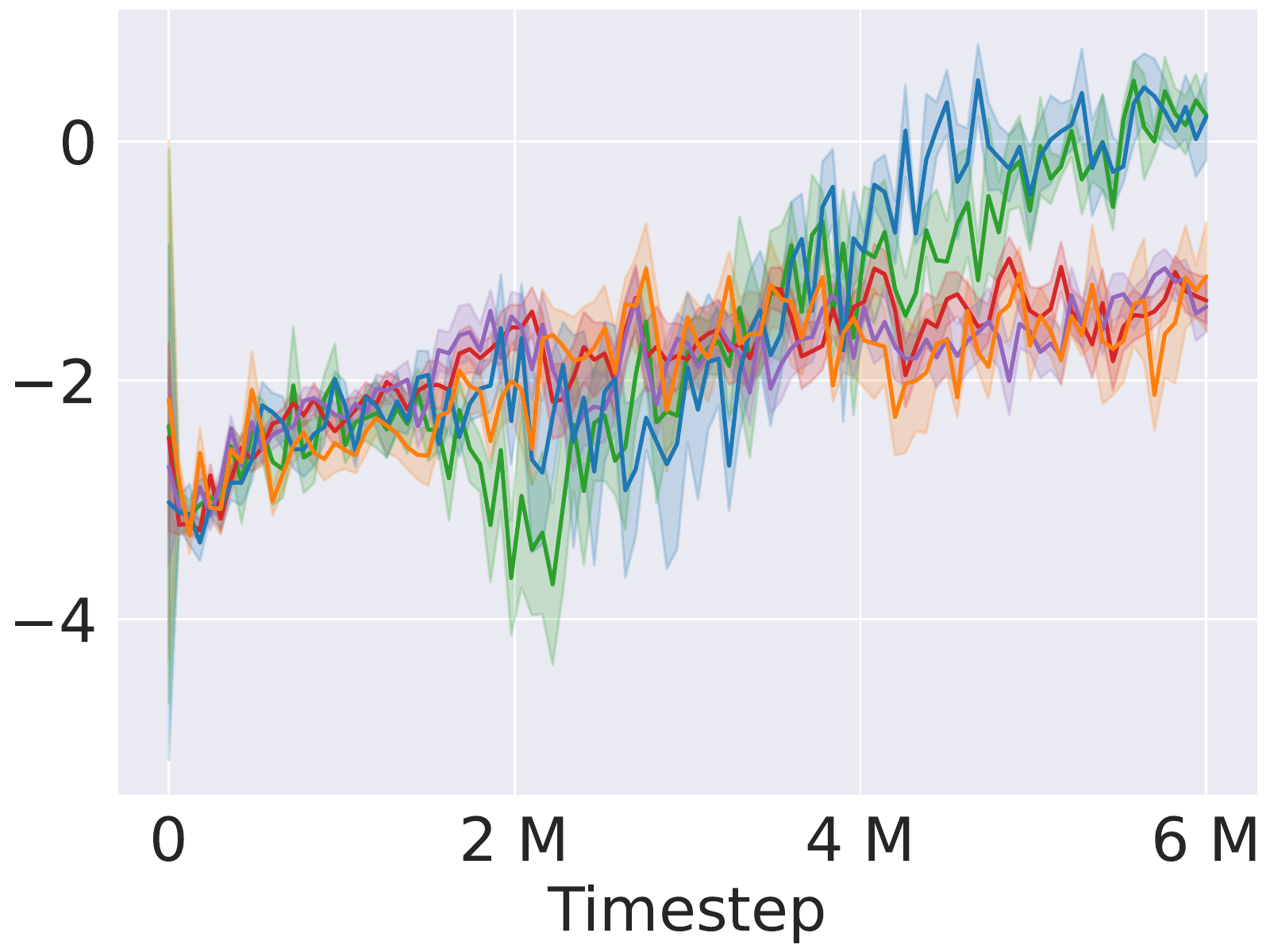}
    \caption{\memoryroomsnine}\label{fig:performance:memoryroomsnine}
  \end{subfigure}
  \caption{Learning performance curves of episodic returns averaged over the
  last $100$ episodes, with statistics computed over $20$ independent runs.
Shaded areas are centered around the empirical mean and show one standard error
of the mean.}\label{fig:performance}
\end{figure*}

\begin{algorithm}[tb]
  \caption{All methods are trained using the same algorithmic structure, just
    using different critics to compute the TD errors $\delta_t$ (see
    \Cref{eq:tderror:h,eq:tderror:s,eq:tderror:hs}). Full episodes are
    iteratively sampled and used for training.  Values $T$ and $E$ vary by
  environment.}\label{alg:code}
  \begin{algorithmic}
    \STATE {\bfseries Input:} max timestep $T$, episodes per gradient step $E$
    \WHILE{timestep $<T$}
      \STATE episodes $\gets$ sample\_episodes($\policy$, $E$)
      \STATE log\_returns(episodes)
      \STATE $\lambda \gets$ negentropy\_schedule(timestep)
      \STATE update $\pparams$ and $\cparams$ via $\nabla\left( \ploss +
      \lambda\hloss \right)$ and $\nabla\closs$
    \ENDWHILE
  \end{algorithmic}
\end{algorithm}

We compare the learning performances of five actor-critic variants.  \ach,
\acs, and \achs\ are respectively (symmetric) A2C with history critic
$\vmodel(h)$, asymmetric A2C with state critic $\vmodel(s)$, and asymmetric A2C
with history-state critic $\vmodel(h, s)$.  To demonstrate that the
environments feature significant partial observability, we include two
``quasi-reactive'' variants of (symmetric) A2C, meaning that they only receive
a fixed number of recent actions and observations.  \achrtwo\ and \achrfour\
respectively receive the latest $2$ and $4$ actions and observations.
We evaluate on $8$ navigation tasks which require different forms of
information gathering and memorization: \heavenhellthree\ and
\heavenhellfour~\cite{bonet_solving_1998,baisero2019gym-pomdps}, \shoppingfive\
and \shoppingsix~\cite{baisero2019gym-pomdps}, \carflag~\cite{nguyen2021penvs},
\cleaner~\cite{jiang_multi-agent_2021}, and \memoryroomsseven\ and
\memoryroomsnine~\cite{baisero2021gym-gridverse}; for details, see
appendix~\cite{baisero_unbiased_2022}.

Each method is trained and evaluated using the same
code\footnote{https://github.com/abaisero/asym-rlpo/} (see \Cref{alg:code}).
Model architectures vary by environment;  for more details, see
\Cref{sec:architectures}.  For each method, we perform a grid-search over
hyper-parameters of interest and select the hyper-parameter combination which
leads to the best performance (prioritizing learning stability over convergence
speed if needed); for more details, see appendix~\cite{baisero_unbiased_2022}.
Each combination of hyper-parameters is evaluated over $20$ independent runs to
guarantee statistical significance.

\subsection{Results and Discussion}

We show two relevant results from our evaluation:
\begin{enumerate*}[label=(\alph*)]
  \item in \Cref{fig:performance}, the empirical learning curve statistics, and
  \item in \Cref{fig:criticvalues}, how critic values change during training for important history-state pairs.
\end{enumerate*}

\subsubsection{Learning Curves}

We first note that the ``quasi-reactive'' baselines perform poorly in most
domains, demonstrating that these control problems feature non-trivial
partial observability which requires information gathering strategies and/or
memorization of the past.  Even in \shoppingfive, where \achrfour\ eventually
manages to reach the performance of other successful methods, its convergence
speed is significantly slower (\Cref{fig:performance:shoppingfive}).
On the other hand, the non-reactive \ach\ either performs much better,
indicating that the additional memory is useful if not necessary
(\Cref{fig:performance:shoppingfive,fig:performance:shoppingsix,fig:performance:cleaner,fig:performance:memoryroomsseven,fig:performance:memoryroomsnine}),
or it also fails, indicating that the task is still challenging even when the
entire history is available, due to representation learning difficulties
(\Cref{fig:performance:heavenhellthree,fig:performance:heavenhellfour,fig:performance:carflag}).

The \acs\ baseline displays a variety of characteristics depending on the
environment, mostly problematic.  While \acs\ managed to achieve competitive
performance in \carflag\ (\Cref{fig:performance:carflag}), in all other cases
it either completely fails to perform the task
(\Cref{fig:performance:heavenhellthree,fig:performance:heavenhellfour,fig:performance:memoryroomsseven,fig:performance:memoryroomsnine}),
or it slowly converges to a sub-optimal behavior
(\Cref{fig:performance:shoppingfive,fig:performance:shoppingsix}).  \cleaner\
in particular demonstrates instability issues, causing the performance to
collapse after a certain point (\Cref{fig:performance:cleaner}).
We argue that the poor convergence performance and learning instability
displayed by \acs\ are two facets of the theoretical issues discussed in
\Cref{sec:aa2c}.  Poor final performance may be easily explained by the
\emph{history-aliasing} issue whereby the state critic model $\vmodel(s)$ may
not be able to correctly evaluate a given history, while instability may be
easily explained by the lack of a well-defined state value function
$\vpolicy(s)$ altogether.

In contrast, our proposed unbiased asymmetric variant \achs\ displays some of
the best learning characteristics across all environments.  In \cleaner,
\memoryroomsseven, and \memoryroomsnine, its performance matches that of \ach\
(\Cref{fig:performance:cleaner,fig:performance:memoryroomsseven,fig:performance:memoryroomsnine}),
while in \carflag\ it matches that of \acs\ (\Cref{fig:performance:carflag}).
In and of itself, this indicates that \achs\ is able to exploit whichever
source of information (history or state) happens to be more suitable in
practice to solve a given task.
On top of that, \achs\ demonstrates \emph{strictly} better final performance
and/or convergence speed than both \ach\ and \acs\ in \shoppingfive\ and
\shoppingsix\
(\Cref{fig:performance:shoppingfive,fig:performance:shoppingsix}),
demonstrating that it is not only able to use the best source of information,
but also of combining both sources to achieve a higher best-of-both-worlds
performance.
This ability is pushed one step further and demonstrated in \heavenhellthree\
and \heavenhellfour, where \achs\ is the \emph{only} method capable of learning
to solve the task at all
(\Cref{fig:performance:heavenhellthree,fig:performance:heavenhellfour}).
These results strongly demonstrate the importance of exploiting asymmetric
information in ways which are theoretically justified and sound, as done in our
work.

\subsubsection{Critic Values}

\begin{figure}[t!]
  \centering
  \begin{subfigure}{.66\linewidth}
    \centering
    \includegraphics[width=\linewidth]{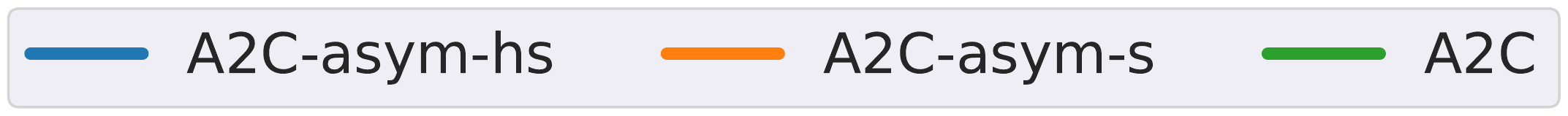}
  \end{subfigure}
  
  \begin{subfigure}{.49\linewidth}
    \centering
    \includegraphics[width=\linewidth]{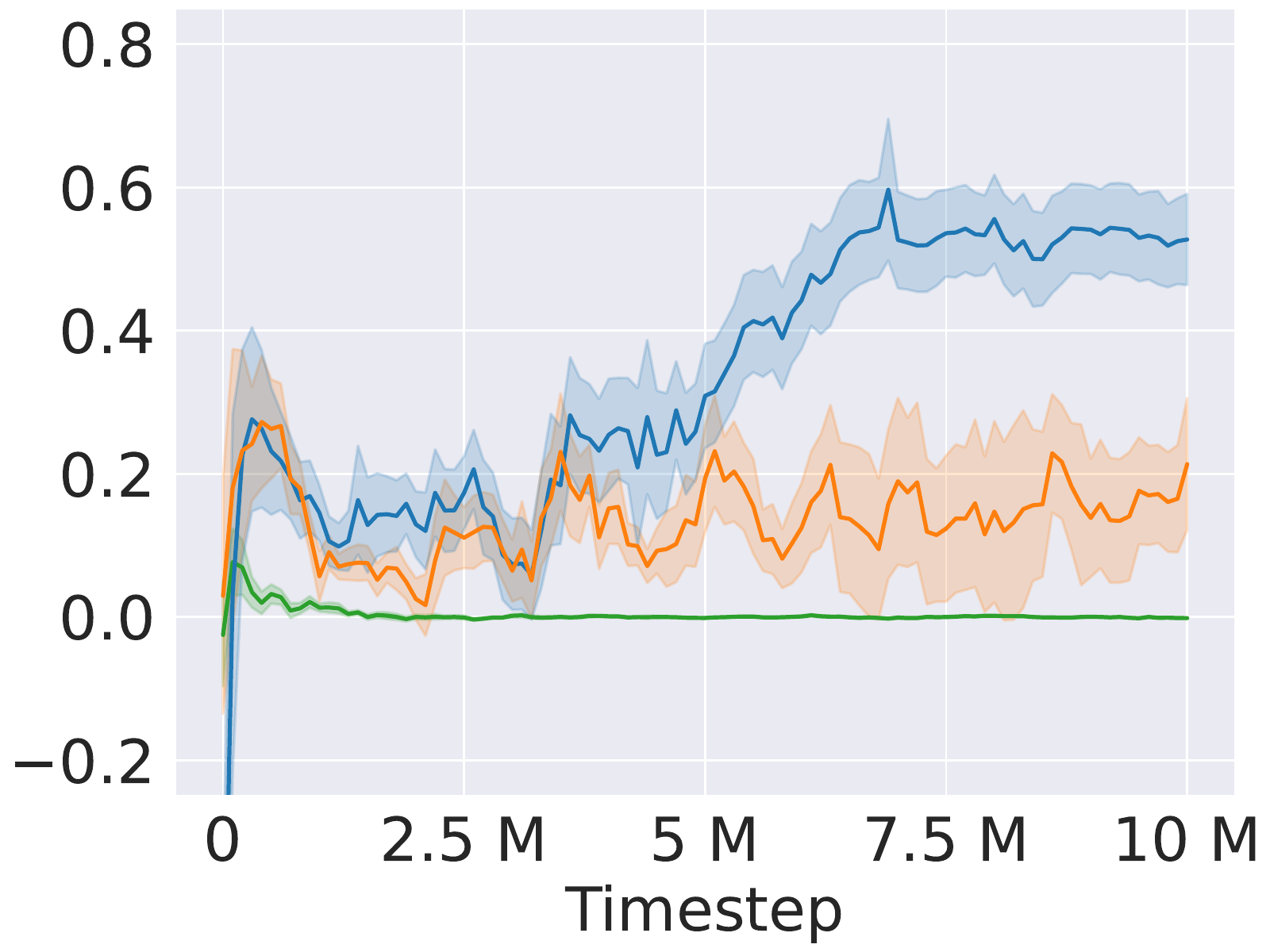}
    \caption{Heaven left, w/o priest visit.}\label{fig:critic:heaven_left_no_priest}
  \end{subfigure}
  \begin{subfigure}{.49\linewidth}
    \centering
    \includegraphics[width=\linewidth]{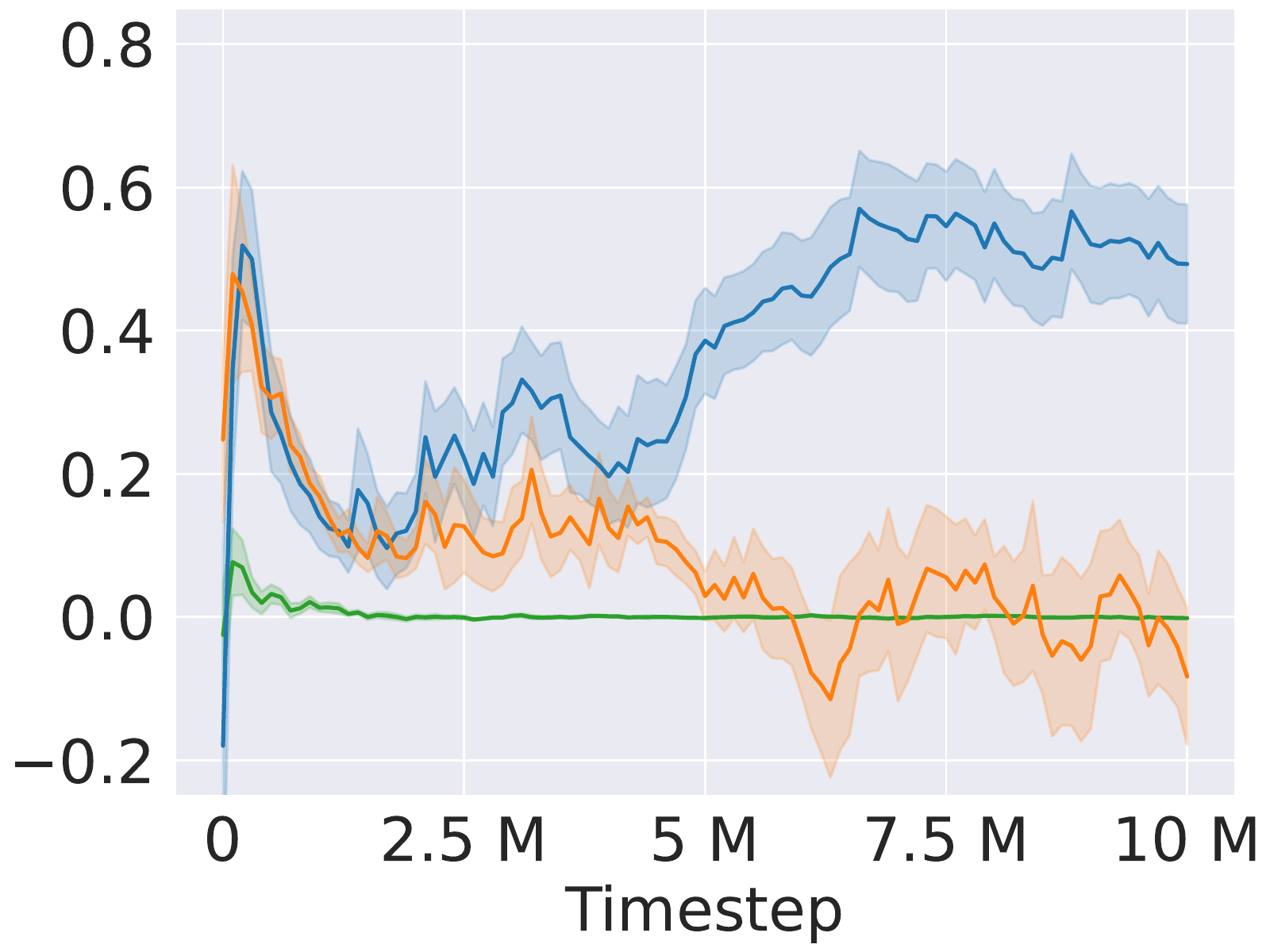}
    \caption{Heaven right, w/o priest visit.}\label{fig:critic:heaven_right_no_priest}
  \end{subfigure}

  \begin{subfigure}{.49\linewidth}
    \centering
    \includegraphics[width=\linewidth]{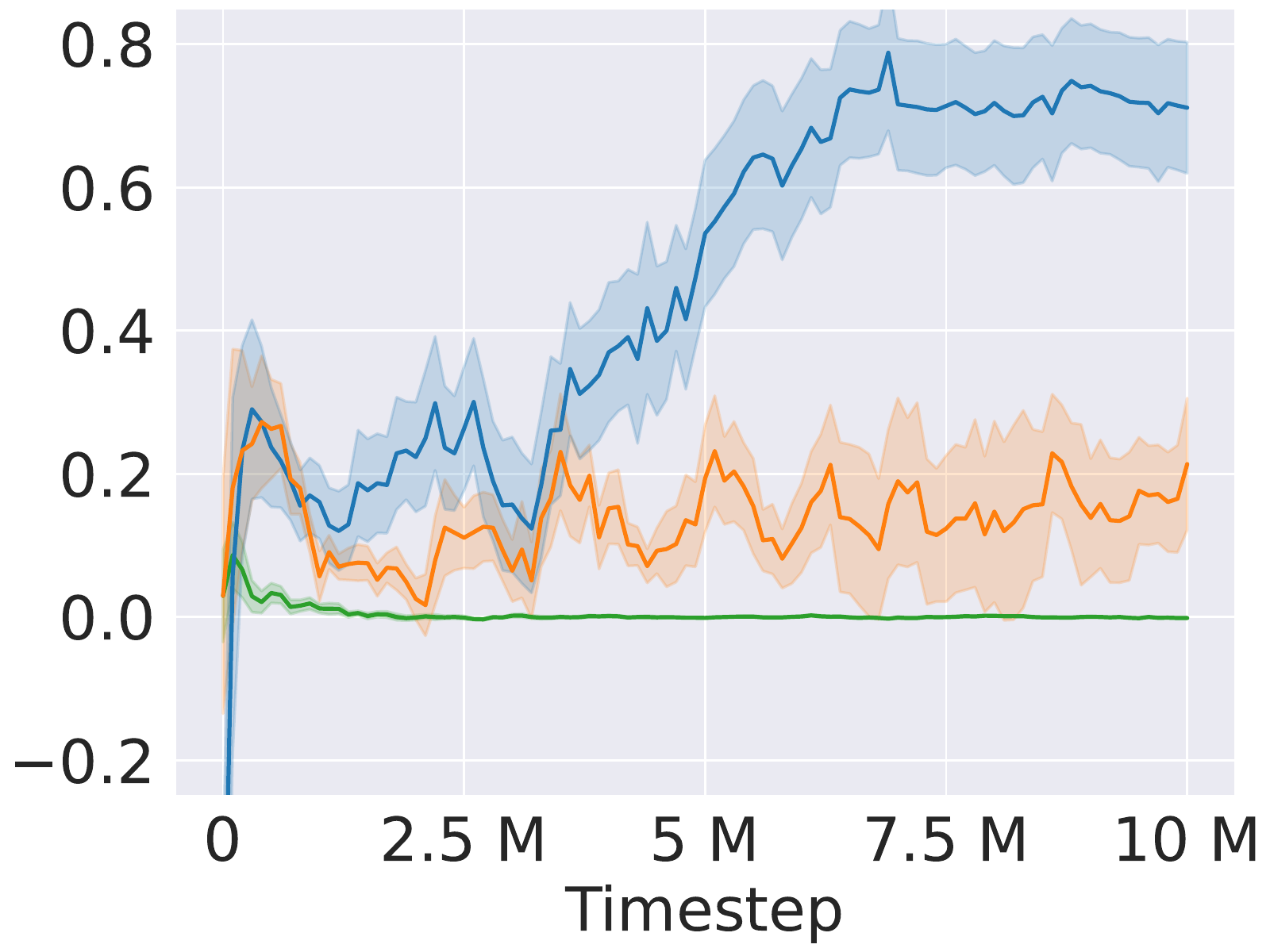}
    \caption{Heaven left, w/ priest visit.}\label{fig:critic:heaven_left_priest}
  \end{subfigure}
  \begin{subfigure}{.49\linewidth}
    \centering
    \includegraphics[width=\linewidth]{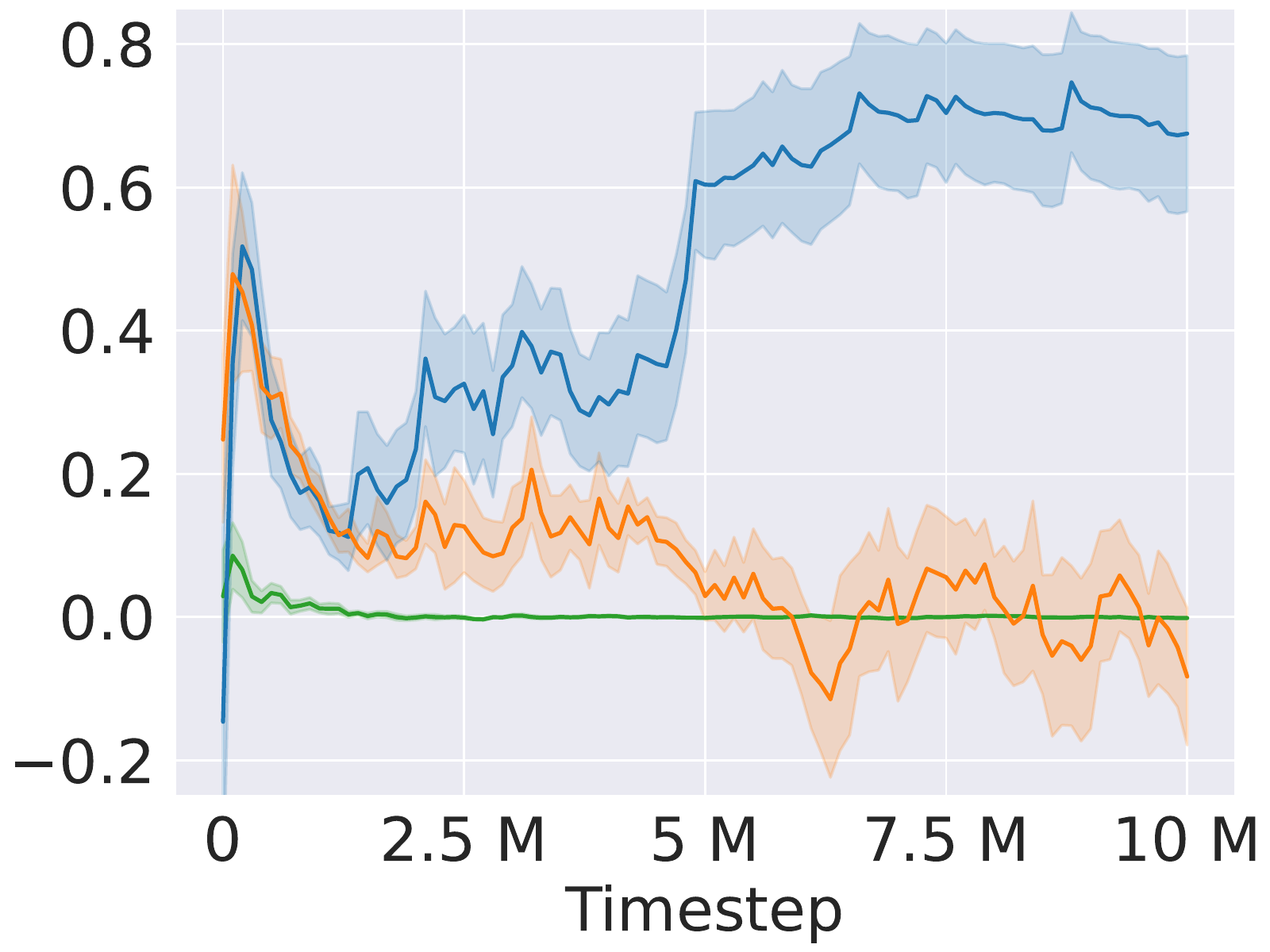}
    \caption{Heaven right, w/ priest visit.}\label{fig:critic:heaven_right_priest}
  \end{subfigure}

  \caption{Critic value statistics for $4$ key history-state pairs in
  \heavenhellfour, evaluated throughout training, with statistics computed over
$20$ independent runs.  Full description in text.}\label{fig:criticvalues}
\end{figure}

To further inspect the behavior of each critic, \Cref{fig:criticvalues} shows
the evolution of critic values over the course of training for important
history-state pairs in \heavenhellfour.  We use $4$ deliberately chosen
history-state pairs which are particularly important in this environment.  In
each case, the agent is located at the fork between \emph{heaven} and
\emph{hell}, and the cases differ by the position of \emph{heaven} (left or
right) and whether the agent has previously performed the information-gathering
sequence of actions necessary to know the position of \emph{heaven} (by
visiting the priest).

Unsurprisingly, we first note that critic values are correlated with the
respective agent's performance (\Cref{fig:performance:heavenhellfour}).  Beyond
that, the critics show certain individual characteristics:  namely, the critics
which focus on a single aspect of the join history-state output the exact same
values for different history-states.  Although hard to see, the \ach\ critic
$\vmodel(h)$ outputs are identical in
\Cref{fig:critic:heaven_left_no_priest,fig:critic:heaven_right_no_priest}, as
those values are associated with the same histories (but not the same states).
Similarly, the \acs\ critic $\vmodel(s)$ outputs are identical in
\Cref{fig:critic:heaven_left_no_priest,fig:critic:heaven_left_priest} and
\Cref{fig:critic:heaven_right_no_priest,fig:critic:heaven_right_priest}
respectively, as those values are associated with the same states (but not the
same histories).  This confirms a straightforward truth: that the state critic
$\vmodel(s)$ is intrinsically unable to differentiave between values associated
to different histories if they happen to be associated with the same state,
which can be particularly detrimental in such information-gathering and memory
dependent tasks.  On the other hand, the \achs\ critic $\vmodel(h, s)$ has the
ability to output different values, as needed, for each of the four cases.
Note, in particular, that the \achs\ critic is able to associate a higher
reward to the agent if it has already performed the information-gathering
actions (\Cref{fig:critic:heaven_left_priest,fig:critic:heaven_right_priest}),
compared to when it has not
(\Cref{fig:critic:heaven_left_no_priest,fig:critic:heaven_right_no_priest}),
which helps the agent determine that the information-gathering actions are
important and should be performed.

\section{Conclusions}

In partially observable control problems, the offline training/online execution
framework offers the peculiar opportunity to access the system's state during
training, which otherwise remains latent during execution.
Asymmetric methods trained offline can potentially exploit such privileged
information to help train the agents to reach better performance and/or train
more efficiently and using less data than before.
While this idea has great potential, current state-of-the-art methods are
motivated and driven by empirical results rather than theoretical analysis.  In
this work, we exposed fundamental theoretical issues with a standard variant of
asymmetric actor-critic which made use of state critics $\vpolicy(s)$, and
proposed an \emph{unbiased} asymmetric variant which makes use of history-state
critics $\vpolicy(h, s)$ and is the first of its kind to be analytically sound
and theoretically justified.  Although this represents a relatively simple
change, its effects are profound, as demonstrated in both theoretical analysis
and empirical results.
Our evaluations confirm our analysis, and demonstrate both the issues with
state-based critics and the benefits of history-state critics in environments
which exhibit significant partial observability.

Although our evaluation only concerns A2C, the same concepts are easily
extensible to other critic-based RL
methods~\cite{silver_deterministic_2014,lillicrap_continuous_2015,degris_off-policy_2012,mnih_asynchronous_2016}.
The potential for future work is varied.  One possibility is to extend the
theory of history-state value functions to optimal value functions $\qopt(h, s,
a)$, and develop theoretically sound asymmetric variants of value-based deep RL
methods such as \emph{DQN}~\cite{mnih_playing_2013}.  Another possibility is to
integrate asymmetric information with state-of-the-art maximum entropy
value/critic-based methods such as \emph{soft
  Q-learning}~\cite{haarnoja_reinforcement_2017}, and \emph{soft
actor-critic}~\cite{haarnoja_soft_2018}.  Finally, another venue for
improvement is to extend our theory and approach to multi-agent methods,
potentially bringing theoretical rigor and improved
performance~\cite{foerster_counterfactual_2018,lowe_multi-agent_2017,li_robust_2019,wang_r-maddpg_2020,yang_cm3_2018,rashid_qmix_2018,mahajan_maven_2019,rashid_weighted_2020}.

%%%%%%%%%%%%%%%%%%%%%%%%%%%%%%%%%%%%%%%%%%%%%%%%%%%%%%%%%%%%%%%%%%%%%%%%

%%% The acknowledgments section is defined using the "acks" environment
%%% (rather than an unnumbered section). The use of this environment 
%%% ensures the proper identification of the section in the article 
%%% metadata as well as the consistent spelling of the heading.
\begin{acks}
  This research was funded by NSF award 1816382.
\end{acks}

%%%%%%%%%%%%%%%%%%%%%%%%%%%%%%%%%%%%%%%%%%%%%%%%%%%%%%%%%%%%%%%%%%%%%%%%

%%% The next two lines define, first, the bibliography style to be 
%%% applied, and, second, the bibliography file to be used.

\balance
\bibliographystyle{ACM-Reference-Format} 
\bibliography{references}

\onecolumn
\appendix
\section{Timed Value Functions}\label{sec:special:timed}

\Cref{sec:aa2c} shows that for a general POMDP and policy, the state value
function $\vpolicy(s)$ is not necessarily well-defined, in part due to issues
caused by the lack of time information.  Here, we consider addressing the
primary issue by providing explicit time-index information via \emph{timed}
value functions, $\vpolicy_t(s)$ and $\qpolicy_t(s, a)$, which represent the
expected returns obtained when the agent finds itself in a state $s$ at time
$t$,
\begin{align}
  \vpolicy_t(s) &= \sum_{a\in\aset} \Pr(A_t=a\mid S_t=s) \qpolicy_t(s, a) \,,
  \label{eq:vts} \\
  \qpolicy_t(s, a) &= R(s, a) + \gamma \Exp_{s'\mid s, a}\left[ \vpolicy_{t+1}(s')
\right] \,. \label{eq:qtsa}
\end{align}

Once again, we analyze the state-dependent action distribution term to verify
correctness, and expand it by integrating over histories;  this time, we can
use the explicit time-index information to integrate over histories of a given
length only,
\begin{equation}
  \Pr(A_t=a\mid S_t=s) = \sum_{h\in\hset_t} \Pr(H_t=h\mid S_t=s) \policy(a; h)
  \,.  \label{eq:sumrule:timed}
\end{equation}

Because \Cref{eq:sumrule:timed} is now restricted to histories of a given
length $t$, the probability term $\Pr(H_t=h\mid S_t=s)$ is well-defined, and in
turn $\Pr(A_t=a\mid S_t=s)$ and $\vpolicy_t(s)$ are likewise well-defined.
Nevertheless, its utility for the purpose of asymmetric reinforcement learning
remains unclear because
\begin{enumerate*}[label=(\alph*)]
  \item it is still not formally proven whether timed value functions are
    unbiased, i.e., whether $\vpolicy_t(h) = \Exp_{s\mid h}\left[ \vpolicy_t(s)
    \right]$, and
  \item it is harder for timed value critics $\vmodel(s, t)$ to generalize
    appropriately across the additional discrete input $t$.
\end{enumerate*}

\section{Additional Lemmas and Proofs}\label{sec:proofs}

\subsection{Differing Histories with Shared Beliefs}

The main proof of \Cref{thm:vs:bias} is based the commonly-known fact (which
could be considered a lemma in its own right) that, in a generic POMDP, two
different histories $h'\neq h''$ may be associated with the same belief $b(h')
= b(h'')$.  This section contains some examples where this happens in some of
the environments used in our evaluation.  To better understand the
environments, and therefore the following examples, see
\Cref{sec:environments}.

\paragraph{Example 1} In Memory-Four-Rooms, consider the following histories:
\begin{itemize}
  \item The agent turns left 4 times until it reaches the initial orientation.
  \item The agent turns right 4 times until it reaches the initial orientation.
  \item The agent turns any direction any number of times such that it ends up
    viewing all possible orientations, and finally reaches the initial
    orientation.
\end{itemize}
In each case, the agent has received the same amount of total information, just
in a different order, which results in the same belief.

\paragraph{Example 2} In Heaven-Hell, consider any sequence of actions which
does not reach an exit, \emph{does not} result in visiting the priest, and
ends with the agent occupying the same final position.  At the end of any such
sequence, the agent has full knowledge of its position, and no additional
knowledge of the position of the door to heaven, i.e., all such sequences
result in the same final belief.

\paragraph{Example 3} In Heaven-Hell, consider any sequence of actions which
does not reach an exit, \emph{does} result in visiting the priest, and ends
with the agent occupying the same final position.  At the end of any such
sequence, the agent has full knowledge of both its position and the position of
the door to heaven, i.e., all such sequences result in the same final belief.

\subsection{Additional Proofs for \Cref{thm:vs:bias}}

This section contains two additional proofs (one sketch and one formal) omitted
from the main body of the document for the case of a reactive policy
(\Cref{sec:special:reactive:po}).  These two proofs complement the more general
one of \Cref{thm:vs:bias}, and take into account the assumptions of a reactive
policy and an observation function which depends only on the current state.

\begin{proof}[Proof (sketch) by contradiction]
  First, we assume that $\vpolicy(s)$ is unbiased and show that $\qpolicy(s,a)$
  (as defined by \Cref{eq:qsa}) is unbiased,
  \begin{align}
    \Exp_{s\mid h}\left[ \qpolicy(s, a) \right] &= \Exp_{s\mid h}\left[
    \rfn(s, a) + \gamma \Exp_{s'\mid s, a}\left[ \vpolicy(s') \right] \right] \nonumber \\
    &= \Exp_{s\mid h}\left[ \rfn(s, a) \right] + \gamma \Exp_{s\mid h}\left[
  \Exp_{s'\mid s, a}\left[ \vpolicy(s') \right] \right] \nonumber \\
    &= \Exp_{s\mid h}\left[ \rfn(s, a) \right] + \gamma \Exp_{s'\mid h,
    a}\left[ \vpolicy(s') \right] \nonumber \\
    &= \Exp_{s\mid h}\left[ \rfn(s, a) \right] + \gamma \Exp_{o\mid h, a}
    \Exp_{s'\mid hao}\left[ \vpolicy(s') \right] \nonumber \\
    &= \rfn(h, a) + \gamma \Exp_{o\mid h, a}\left[ \vpolicy(hao) \right]
    \nonumber \\
    &= \qpolicy(h, a) \,.
  \end{align}
  Next, we show that even if $\qpolicy(s, a)$ is unbiased, $\vpolicy(s)$ (as
  defined by \Cref{eq:vs}) is biased, which contradicts the original
  assumption.  To do that, we expand 
  the expected state value function $\Exp_{s\mid h}\left[ \vpolicy(s) \right]$
  and the history value function $\vpolicy(h)$, and show that there is a
  concrete difference between them:
  \begin{align}
    \Exp_{s\mid h}\left[ \vpolicy(s) \right] &= \Exp_{s\mid h}\left[
    \sum_{a\in\aset} \Pr(a\mid s) \qpolicy(s, a) \right] \nonumber \\
    &= \Exp_{s\mid h}\left[ \sum_{a\in\aset} \Exp_{o\mid s}\left[
    \pi(a; o) \right] \qpolicy(s, a) \right] \,, \label{eq:proof:vs} \\
    \vpolicy(h) &= \sum_{a\in\aset} \pi(a; o_h) \qpolicy(h, a) \nonumber \\
    &= \sum_{a\in\aset} \pi(a; o_h) \Exp_{s\mid h}\left[ \qpolicy(s, a)
    \right] \nonumber \\
    &= \Exp_{s\mid h}\left[ \sum_{a\in\aset} \Exp_{o\mid s}\left[ \pi(a; o_h)
    \right] \qpolicy(s, a) \right] \,. \label{eq:proof:vh}
  \end{align}
  \Cref{eq:proof:vs,eq:proof:vh} differ in terms of which observation is used
  by the policy;  in \Cref{eq:proof:vs}, an observation $o$ \emph{inferred}
  from a state $s$ \emph{inferred} from the history $h$ is used, while in
  \Cref{eq:proof:vh} the final observation $o_h$ of the history $h$ is used.
  These two observations $o$ and $o_h$ are not generally the same, and the
  respective expectations are similarly not generally the same.
  The nested expectation in \Cref{eq:proof:vs} can be interpreted as a
  \emph{lossy} round-trip inference from history to state and from state back to
  observation $h\to s\to o$.  Although histories and states tend to be somewhat
  correlated, both state aliasing and history aliasing make the roundtip
  conversion imperfect, causing a mismatch between the expected state value
  function $\Exp_{s\mid h}\left[ \vpolicy(s) \right]$ and the history value
  function $\vpolicy(h)$ in the general control case of a general POMDP.
\end{proof}

\begin{proof}[Proof by example]
  This is a proof by example (with a proof by contradiction element).  We will
  define the \emph{good/bad} POMDP and, for a specific policy and history,
  first calculate $\Exp_{s\mid h}\left[ \vpolicy(s) \right]$ exactly, and then
  $\vpolicy(h)$ using bootstrapping (while also assuming $\vpolicy(hao) =
  \Exp_{s'\mid hao}\left[ \vpolicy(s') \right]$).  We show that the two values
  are numerically different.

  In the \emph{good/bad} POMDP, $\sset = \{\good, \bad \}, \aset = \{ \good,
  \bad \}$, $\oset = \{ \good, \bad \}$;  At times, we will use the shorthands
  $\GGG$ and $\B$.  The initial state distribution is uniform, and each state
  deterministically transitions into itself.  The \good\ state always emits the
  \good\ observation, while the \bad\ state emits a random observation.
  Consider the reward function such that $R(s, a) = \Ind\left[ a = \good
  \right]$, i.e., the agent receives a reward whenever it choses the \good\
  action.
  We will denote a history as the concatenation of alternating observations and
  actions, starting with an observation.  To keep the notation compact, we will
  occasionally use symbols \GGG\ and \B\ to represent \good\ and \bad\ states,
  observations and actions.
  Consider a deterministic policy $\pi(a; h) = \Ind\left[ a = o_h \right]$
  which returns the action corresponding to the last observation.  Note that
  this POMDP and this policy satisfy the requirements to guarantee that
  $\vpolicy(s)$ is well defined.

  Next, we calculate the state values $\vpolicy(s)$.  The $\good$ state always emits
  the $\good$ observation, so the agent will always choose the $\good$ action
  and receive a reward of $1$, then the state will always transition into
  itself,
  \begin{align}
    \vpolicy(s=\good) &= 1 + \gamma \vpolicy(s=\good) \nonumber \\
    &= \frac{1}{1 - \gamma} \,.
  \end{align}

  On the other hand, the $\bad$ state will only emit the $\good$ observation
  half of the times, so the agent will only choose the $\good$ action and
  receive a reward of $1$ half of the times, then the state will always
  transition into itself,
  \begin{align}
    \vpolicy(s=\bad) &= \frac{1}{2} + \gamma \vpolicy(s=\bad) \nonumber \\
    &= \frac{1}{2(1 - \gamma)} \,.
  \end{align}

  Next, we consider the history $h=\GGG$ after a single initial \good\
  observation, and calculate the history value $\vpolicy(h)$.  Before proceeding,
  we need to calculate a few intermediate quantities, such as the
  belief-distribution:
  \begin{align}
    \Pr(s=\GGG\mid h=\GGG) &\propto \Pr(h=\GGG\mid s=\GGG) \Pr(s=\GGG) \nonumber \\
    &= 1 \frac{1}{2} \nonumber \\
    &= \frac{1}{2} \,, \\
    \Pr(s=\B\mid h=\GGG) &\propto \Pr(h=\GGG\mid s=\B) \Pr(s=\B) \nonumber \\
    &= \frac{1}{2} \frac{1}{2} \nonumber \\
    &= \frac{1}{4} \,, \\
    \intertext{therefore}
    \Pr(s=\GGG\mid h=\GGG) &= \frac{2}{3} \,, \\
    \Pr(s=\B\mid h=\GGG) &= \frac{1}{3} \,.
  \end{align}

  We also calculate the belief-state distribution after two other histories.
  First $h=\GGG\GGG\GGG$,
  \begin{align}
    \Pr(s=\GGG\mid h=\GGG\GGG\GGG) &\propto \Pr(h=\GGG\GGG\GGG\mid s=\GGG) \Pr(s=\GGG) \nonumber \\
    &= 1 \frac{1}{2} \nonumber \\
    &= \frac{1}{2} \,, \\
    \Pr(s=\B\mid h=\GGG\GGG\GGG) &\propto \Pr(h=\GGG\GGG\GGG\mid s=\B) \Pr(s=\B) \nonumber \\
    &= \frac{1}{4} \frac{1}{2} \nonumber \\
    &= \frac{1}{8} \,, \\
    \intertext{therefore}
    \Pr(s=\GGG\mid h=\GGG\GGG\GGG) &= \frac{4}{5} \,, \\
    \Pr(s=\B\mid h=\GGG\GGG\GGG) &= \frac{1}{5} \,.
  \end{align}
  Then $h=\GGG\GGG\B$,
  \begin{align}
    \Pr(s=\GGG\mid h=\GGG\GGG\B) &\propto \Pr(h=\GGG\GGG\B\mid s=\GGG) \Pr(s=\GGG) \nonumber \\
    &= 0 \frac{1}{2} \nonumber \\
    &= 0 \,, \\
    \Pr(s=\B\mid h=\GGG\GGG\B) &\propto \Pr(h=\GGG\GGG\B\mid s=\B) \Pr(s=\B) \nonumber \\
    &= \frac{1}{4} \frac{1}{2} \nonumber \\
    &= \frac{1}{8} \,, \\
    \intertext{therefore}
    \Pr(s=\GGG\mid h=\GGG\GGG\B) &= 0 \,, \\
    \Pr(s=\B\mid h=\GGG\GGG\B) &= 1 \,.
  \end{align}

  We also need to calculate the observation emission probabilities,
  \begin{align}
    \Pr(o=\GGG\mid h=\GGG, a=\GGG) &= \Pr(s=\GGG\mid h=\GGG) \Pr(o=\GGG\mid s=\GGG) \nonumber \\
    &\hphantom{{}={}} + \Pr(s=\B\mid h=\GGG) \Pr(o=\GGG\mid s=\B) \nonumber \\
    &= \frac{2}{3} 1 + \frac{1}{3} \frac{1}{2} \nonumber \\
    &= \frac{5}{6} \,, \\
    \Pr(o=\B\mid h=\GGG, a=\GGG) &= \Pr(s=\GGG\mid h=\GGG) \Pr(o=\B\mid s=\GGG) \nonumber \\
    &\hphantom{{}={}} + \Pr(s=\B\mid h=\GGG) \Pr(o=\B\mid s=\B) \nonumber \\
    &= \frac{2}{3} 0 + \frac{1}{3} \frac{1}{2} \nonumber \\
    &= \frac{1}{6} \,.
  \end{align}

  Next, we calculate $\vpolicy(h=\GGG)$ under the assumption that the equality
  holds,
  \begin{align}
    \vpolicy(h=\GGG) &= \Exp_{s\mid h=\GGG}\left[ \vpolicy(s) \right] \nonumber \\
    &= \Pr(s=\GGG\mid h=\GGG) \vpolicy(s=\GGG) + \Pr(s=\B\mid h=\GGG) \vpolicy(s=\B) \nonumber \\
    &= \frac{2}{3} \frac{1}{1-\gamma} + \frac{1}{3} \frac{1}{2(1-\gamma)} \nonumber \\
    &= \frac{5}{6(1-\gamma)} \label{eq:proof:value:1} \,.
  \end{align}

  We can also apply the equality to other histories,
  \begin{align}
    \vpolicy(h=\GGG\GGG\GGG) &= \Exp_{s\mid h=\GGG\GGG\GGG}\left[ \vpolicy(s) \right] \nonumber \\
    &= \Pr(s=\GGG\mid h=\GGG\GGG\GGG) \vpolicy(s=\GGG) + \Pr(s=\B\mid h=\GGG\GGG\GGG)
    \vpolicy(s=\B) \nonumber \\
    &= \frac{4}{5} \frac{1}{1-\gamma} + \frac{1}{5} \frac{1}{2(1-\gamma)} \nonumber \\
    &= \frac{9}{10(1-\gamma)} \,, \\
    \vpolicy(h=\GGG\GGG\B) &= \Exp_{s\mid h=\GGG\GGG\B}\left[ \vpolicy(s) \right] \nonumber \\
    &= \Pr(s=\GGG\mid h=\GGG\GGG\B) \vpolicy(s=\GGG) + \Pr(s=\B\mid h=\GGG\GGG\B)
    \vpolicy(s=\B) \nonumber \\
    &= 0 \frac{1}{1-\gamma} + 1 \frac{1}{2(1-\gamma)} \nonumber \\
    &= \frac{1}{2(1-\gamma)} \,.
  \end{align}

  Next, we calculate $\vpolicy(h=\GGG)$, this time by bootstrapping first, and
  then using the equality.  Note that with the given history $h=\GGG$, the agent
  will choose action $a=\good$.  Then,
  \begin{align}
    \vpolicy(h=\GGG) &= R(h=\GGG, a=\GGG) + \gamma \Exp_{o\mid h=\GGG, a=\GGG}\left[
    \vpolicy(hao=\GGG\GGG o) \right] \nonumber \\
    &= 1 + \gamma \left(
      \begin{aligned}
        \Pr(o=\GGG\mid h=\GGG, a=\GGG) \vpolicy(hao=\GGG\GGG\GGG) \\
        + \Pr(o=\B\mid h=\GGG, a=\GGG) \vpolicy(hao=\GGG\GGG\B)
      \end{aligned} \right) \nonumber \\
    &= 1 + \gamma \frac{5}{6} \frac{9}{10(1-\gamma)} + \gamma \frac{1}{6}
    \frac{1}{2(1-\gamma)} \nonumber \\
    &= \frac{60-60\gamma}{60(1-\gamma)} + \frac{45\gamma}{60(1-\gamma)} + 
    \frac{5\gamma}{60(1-\gamma)} \nonumber \\
    &= \frac{60-10\gamma}{60(1-\gamma)} \nonumber \\
    &= \frac{6-\gamma}{6(1-\gamma)} \,. \label{eq:proof:value:2}
  \end{align}
  The values from \Cref{eq:proof:value:1,eq:proof:value:2} contradict
  each other, therefore, for this POMDP, policy, and history, $\vpolicy(h) \neq
  \Exp_{s\mid h}\left[ \vpolicy(s) \right]$.
\end{proof}

\section{Environments}\label{sec:environments}

In this section, we present a detailed description of each control problem.
While all problems are controlled by categorical actions, they can be split
into three groups based on the types of state and observation representations
provided to the agent:
\begin{itemize}
  \item \heavenhellthree, \heavenhellfour, \shoppingfive, and \shoppingsix\ are
    \emph{categorical} POMDPs,
  \item \carflag\ and \cleaner\ are \emph{feature-vector} POMDPs,
  \item \memoryroomsseven\ and \memoryroomsnine\ are \emph{gridverse} POMDPs.
\end{itemize}

\subsection{Categorical POMDPs}

\newcommand\figfactor{.25}
\newcommand\subfigfactor{1.0}
\begin{figure}[t!]
  \centering

  \begin{subfigure}{\figfactor\linewidth}
    \centering
    \fbox{\resizebox{\subfigfactor\linewidth}{!}{\begin{tikzpicture}

  \tikzset{wall/.style={ultra thick}}
  \tikzset{edge/.style={->, very thick, shorten >=2pt}}

  \draw[step=1.0, wall] (-3,0) grid (4,1);
  \draw[step=1.0, wall] (0,1) grid (1,-3);
  \draw[step=1.0, wall] (0,-3) grid (4,-4);

  \node at (3.5, 0.5) {Exit};
  \node at (-2.5, 0.5) {Exit};

  \node at (0.5, -2.5) {Agent};
  \node at (3.5, -3.5) {Priest};

\end{tikzpicture}}}
    \caption{\heavenhellthree}
  \end{subfigure}
  \qquad\qquad
  \begin{subfigure}{\figfactor\linewidth}
    \centering
    \fbox{\resizebox{\subfigfactor\linewidth}{!}{\begin{tikzpicture}

  \tikzset{wall/.style={ultra thick}}
  \tikzset{edge/.style={->, very thick, shorten >=2pt}}

  \draw[step=1.0, wall] (-4,0) grid (5,1);
  \draw[step=1.0, wall] (0,1) grid (1,-4);
  \draw[step=1.0, wall] (0,-4) grid (5,-5);

  \node at (4.5, 0.5) {Exit};
  \node at (-3.5, 0.5) {Exit};

  \node at (0.5, -3.5) {Agent};
  \node at (4.5, -4.5) {Priest};

\end{tikzpicture}}}
    \caption{\heavenhellfour}
  \end{subfigure}

  \caption{Layout of the Heaven-Hell environments.}\label{fig:map:heavenhell}
\end{figure}
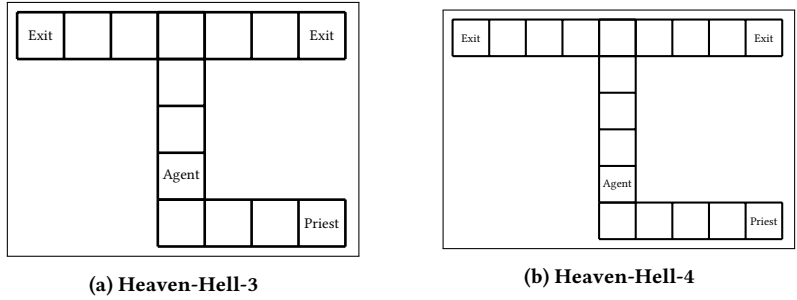

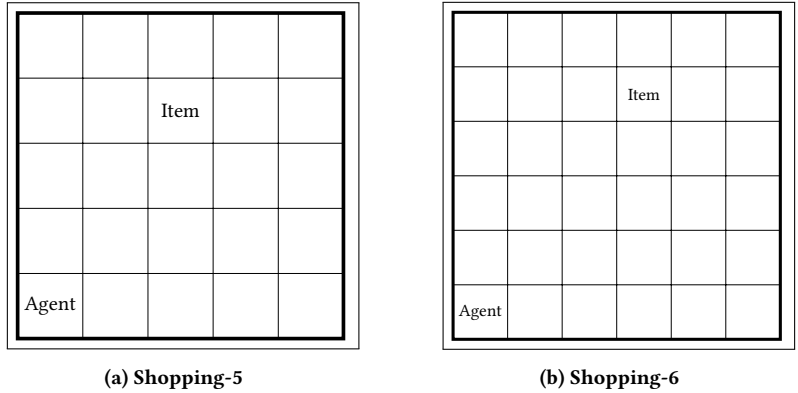
\begin{figure}[t!]
  \centering

  \begin{subfigure}{\figfactor\linewidth}
    \centering
    \fbox{\resizebox{\subfigfactor\linewidth}{!}{\begin{tikzpicture}

  \tikzset{wall/.style={ultra thick}}

  \draw[step=1.0] (0,0) grid (5,5);
  \draw[wall] (0,0) rectangle (5,5);

  \node at (0.5, 0.5) {Agent};
  \node at (2.5, 3.5) {Item};

\end{tikzpicture}}}
    \caption{\shoppingfive}
  \end{subfigure}
  \qquad\qquad
  \begin{subfigure}{\figfactor\linewidth}
    \centering
    \fbox{\resizebox{\subfigfactor\linewidth}{!}{\begin{tikzpicture}

  \tikzset{wall/.style={ultra thick}}

  \draw[step=1.0] (0,0) grid (6,6);
  \draw[wall] (0,0) rectangle (6,6);

  \node at (0.5, 0.5) {Agent};
  \node at (3.5, 4.5) {Item};

\end{tikzpicture}}}
    \caption{\shoppingsix}
  \end{subfigure}

  \caption{Layout of the Shopping environments.}\label{fig:map:shopping}
\end{figure}

The implementation of the categorical POMDPs (and their POMDP files) can be
found in~\cite{baisero2019gym-pomdps}.
In the \emph{categorical} POMDPs, states, actions and observations are all
encoded by categorical indices which have no inherent metric, and which are
intrinsically equally (dis)similar to each other.
While it is not possible to generalize across states and observations via
feature extraction, the primary challenge in these POMDPs is that of
generalizing across different histories.
Because the categorical POMDPs are finite, their state, action and observation
spaces have well-defined sizes, shown in \Cref{tab:environments};  note,
however, that the size of the state space is not a significant measure of the
complexity of partially observable tasks, while the time required to solve the
task (i.e., history length) is a more relevant measure.  While these types of
POMDPs often do not look as impressive as others, due to the unimpressive
categorical representations, they pose unique learning challenges which tend to
focus on the core of decision making, rather than mere feature extraction.

\begin{table}
  \centering
  \caption{Categorical environment properties.}\label{tab:environments}
  \begin{tabular}{lrccc}
    \toprule
    Domain & $|\sset|$ & $|\aset|$ & $|\oset|$ & $\gamma$ \\
    \midrule
    \shoppingfive & 625 & 6 & 50 & 0.99 \\
    \shoppingsix & 1296 & 6 & 72 & 0.99 \\
    \heavenhellthree & 28 & 4 & 15 & 0.99 \\
    \heavenhellfour & 36 & 4 & 19 & 0.99 \\
    \bottomrule
  \end{tabular}
\end{table}

\subsubsection{Heaven-Hell}

In \heavenhell~\cite{bonet_solving_1998}, the agent navigates a corridor-like
gridworld composed of a fork and $3$ dead-ends.  Two dead-ends are exits which
lead to \emph{heaven} or \emph{hell}, although the agent does not know which is
which, while the third dead-end leads to a \emph{priest} who can help the agent
identify the \emph{heaven} exit.
\Cref{fig:map:heavenhell} depicts the gridworlds encoded by \heavenhellthree\
and \heavenhellfour.

\paragraph{States and Observations}
States encode the position of the agent \emph{and} the position of the exit to
heaven.  Observations encode the position of the agent \emph{or} the
position of the exit to heaven.

\paragraph{Actions}
Each time-step, the agent must choose an action from the set \{ \textbf{NORTH},
\textbf{SOUTH}, \textbf{EAST}, \textbf{WEST} \}.  If the agent is at the
priest, it observes \emph{heaven's} location, otherwise it observes its own
position.
To solve the task, the agent needs to \emph{navigate} to the priest, then back
to the \emph{fork}, and on to \emph{heaven}.

\paragraph{Rewards}
The agent receives a sparse reward signal composed of:
\begin{itemize}
  \item a reward of $1.0$ for exiting to \emph{heaven}; and
  \item a reward of $-1.0$ for exiting to \emph{hell}.
\end{itemize}

\subsubsection{Shopping}

\shopping\ simulates an agent going to a shop to buy an item it forgot.
The agent navigates a $5\times 5$ or $6\times 6$ gridworld trying to
\emph{locate} and \emph{select} a randomly positioned item.  The agent's
position is fully observable, while the item's position is only observed when
\emph{queried}.
\Cref{fig:map:shopping} depicts the gridworlds encoded by \shoppingfive\ and
\shoppingsix.

\paragraph{States and Observations}
States encode the position of the agent \emph{and} the position of the item in a
single integer. Observations encode the position of the agent \emph{or}
the position of the item in a single integer.

\paragraph{Actions}
Each time-step, the agent must choose an action from the set \{ \textbf{LEFT},
\textbf{RIGHT}, \textbf{UP}, \textbf{DOWN}, \textbf{QUERY}, \textbf{BUY} \}.
If the agent chooses the \textbf{QUERY} action, it observes the item's position,
otherwise it observes its own position.
To solve the task optimally, the agents needs to \emph{query} the item's
position and remember it, \emph{navigate} to it, and then \emph{buy} it.

\paragraph{Rewards}
The agent receives the following reward signal:
\begin{itemize}
  \item a reward of $-1.0$ for moving;
  \item a reward of $-2.0$ for performing a \textbf{QUERY} action;
  \item a reward of $-5.0$ for performing a \textbf{BUY} action in the wrong cell; and
  \item a reward of $10.0$ for performing a \textbf{BUY} action in the correct cell.
\end{itemize}

\subsection{Feature-Vector POMDPs}\label{sec:environments:extra}

In the \emph{feature-vector} POMDPs, states and observations are provided as
concise feature vectors where each dimension represents a particular aspect of
the state/observation, e.g., in navigation tasks, one dimension could represent
the horizontal position of the agent.  Typically, the observation feature
vectors are obtained by dropping selected dimensions from the state feature
vector.

\begin{figure}[h!]
  \centering

  \begin{subfigure}{.3\linewidth}
    \centering
    \fbox{\includegraphics[width=\linewidth]{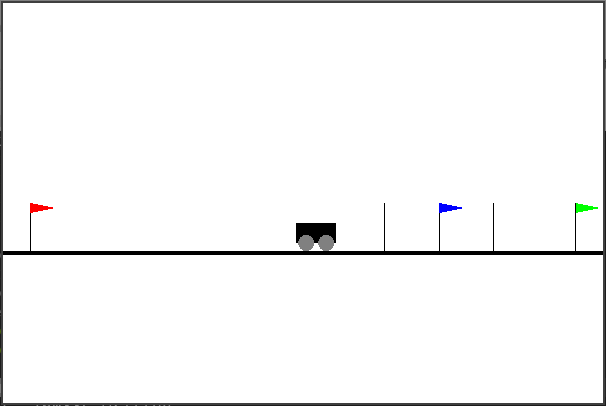}}
    \caption{\carflag}\label{fig:screenshot:carflag}
  \end{subfigure}
  \quad\quad
  \begin{subfigure}{.2\linewidth}
    \centering
    \fbox{\includegraphics[width=\linewidth]{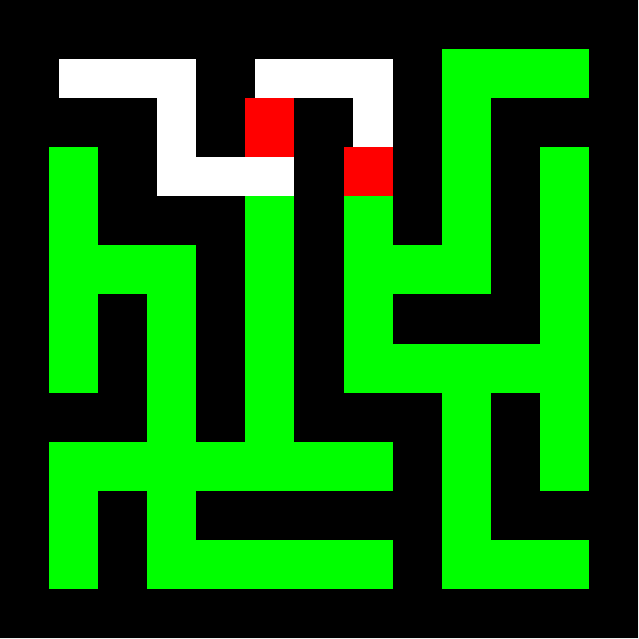}}
    \caption{\cleaner}\label{fig:screenshot:cleaner}
  \end{subfigure}
  
  \caption{Screenshots of the \carflag\ and \cleaner\ environments.}
\end{figure}

\subsubsection{Car-Flag}

The implementation of \carflag\ can be found in~\cite{nguyen2021penvs}.
The agent uses force-control on a car moving along a 1-dimensional axis.  On
opposite extremes are a \emph{good} and a \emph{bad} flags which, if reached,
respectively end the episode positive and negatively.  Along the line, a third
\emph{info} flag appears, which allows the agent to observe the position of the
\emph{good} flag.  \Cref{fig:screenshot:carflag} depicts \carflag.  This setup
is similar to Heaven-Hell, although with force-control and the position of the
\emph{info} flag being significant differences.

\paragraph{States and Observations}

States and observations are both three-dimensional vectors.  In the states
features, the first two dimensions respectively contain the agent position and
velocity, while the third dimension contains the position of the \emph{good}
flag.  The observations are analogous, with the difference that the third
dimension is masked out (set to zero) when the agent is not within range of the
\emph{info} flag.

\paragraph{Actions}

Each time-step, the agent must choose how to accelerate the car using one of
$7$ possible actions representing \{ \textbf{LEFT\_HIGH},
  \textbf{LEFT\_MEDIUM}, \textbf{LEFT\_LOW}, \textbf{NONE},
  \textbf{RIGHT\_HIGH}, \textbf{RIGHT\_MEDIUM}, \textbf{RIGHT\_LOW} \}.

\paragraph{Rewards}

The agent receives a sparse reward signal composed of:
\begin{itemize}
  \item a reward of $1.0$ for reaching the \emph{good} flag; and
  \item a reward of $-1.0$ for reaching the \emph{bad} flag.
\end{itemize}

\subsubsection{Cleaner}

\cleaner~\cite{jiang_multi-agent_2021} is originally a 2-agent environment;
however, for the purpose of our evaluation, we frame it as a single-agent
control problem via fully centralized training and execution.  As a result, the
problem's actions and observations are obstained via Cartesian product of the
respective actions and observations of each separate agent.  In \cleaner, two
robots must cover the entire area of a $13\times 13$ maze-like environment in
order to clean it.  The task is complete when every cell in the grid has been
visited by at least one agent.  The environment is depicted in
\Cref{fig:screenshot:cleaner}.

\paragraph{States and Observations}

The state is provided as a $13\times 13\times 5$ binary tensor, indicating, for
each position in the grid, whether it contains a wall, a dirty cell, a clean
cell, the first agent, or the second agent.  Each agent's observation is given
as a $3\times 3\times 3$ binary tensor encoding the $3\times 3$ area
surrounding the agent.

\paragraph{Actions}

Each agent can move in one of the four directions using actions \{
\textbf{LEFT}, \textbf{RIGHT}, \textbf{UP}, \textbf{DOWN} \}.  In the
centralized control version of this problem, this results in $16$ possible
actions.

\paragraph{Rewards}

In each time-step, a unit reward is given for each new tile cleaned, resulting
in three possible rewards:  $0.0$, $1.0$, and $2.0$.  

\subsection{Gridverse POMDPs}\label{sec:environments:gridverse}

\begin{figure}
  \centering
  \begin{subfigure}{.25\linewidth}
    \centering
    \fbox{\includegraphics[width=\linewidth]{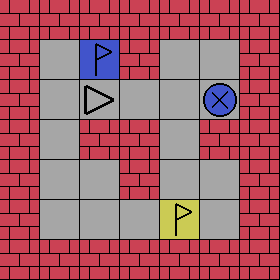}}
    \caption{State.}\label{fig:screenshot:memoryroomsseven:state}
  \end{subfigure}
  \qquad\qquad
  \begin{subfigure}{.25\linewidth}
    \centering
    \fbox{\includegraphics[width=\linewidth]{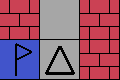}}
    \caption{Observation.}\label{fig:screenshot:memoryroomsseven:observation}
  \end{subfigure}
  \caption{\memoryroomsseven. Note that the states and observations are not
  provided as images to the agent;  these are visualizations for human
understanding.}\label{fig:screenshot:memoryroomsseven}
\end{figure}

\begin{figure}
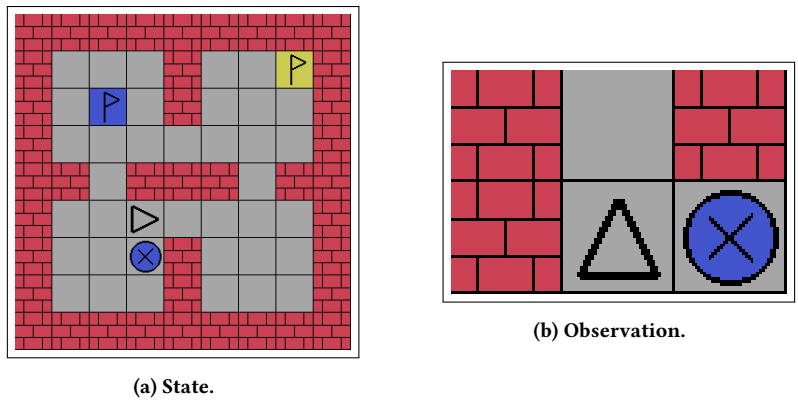

  \centering
  \begin{subfigure}{.25\linewidth}
    \centering
    \fbox{\includegraphics[width=\linewidth]{{figures/env.memory_four_rooms_9x9.state}.png}}
    \caption{State.}\label{fig:screenshot:memoryroomsnine:state}
  \end{subfigure}
  \qquad\qquad
  \begin{subfigure}{.25\linewidth}
    \centering
    \fbox{\includegraphics[width=\linewidth]{{figures/env.memory_four_rooms_9x9.observation}.png}}
    \caption{Observation.}\label{fig:screenshot:memoryroomsnine:observation}
  \end{subfigure}
  \caption{\memoryroomsnine. Note that the states and observations are not
  provided as images to the agent;  these are visualizations for human
understanding.}\label{fig:screenshot:memoryroomsnine}
\end{figure}

The implementation of the gridverse POMDPS can be found
in~\cite{baisero2021gym-gridverse}.
In the \emph{gridverse} POMDPs, actions are still encoded by categorical
indices, while states and observations are encoded as structures which do have
an inherent similarity metric; they are split into different components, some
of which are only available in the state, or which take a different form in the
observation:
\begin{itemize}
  \item A \emph{grid} component: a $3\times H\times W$ volume of categorical
    indices which encode cell type, cell color, and cell status.
    The observation grid component is a slice of the corresponding state grid
    component made to match the agent's perspective:  it is rotated to be a
    first-person view, as shown in
    \Cref{fig:screenshot:memoryroomsseven:observation,fig:screenshot:memoryroomsnine:observation},
    and cells hidden behind walls, if within the observation slice, are
    occluded.
  \item An \emph{agent\_id\_grid} component: a $H\times W$ binary matrix which
    encode the agent's position.  This is only available in the state.
  \item An \emph{agent} component: a $3$-dimensional array of categorical
    indices representing the agent position and orientation.
    The position and orientations in the state agent component are absolute,
    while those in the observation component are relative to the agent's
    perspective---they are essentially constant, and not necessary for control.
  \item An \emph{item} component:  a $3$-dimensional array of categorical
    indices representing the item held by the agent, if any.  While some
    \emph{gridverse} tasks do involve manipulation of items, the ones used in
    our evaluation do not, and this component could potentially be ignored.
\end{itemize}

\subsubsection{Memory-Four-Rooms}

The agent navigates a $7\times 7$ or $9\times 9$ gridworld split into four
rooms;  randomly positioned are a \emph{good} exit, a \emph{bad} exit, and a
\emph{beacon} with the same color as the \emph{good} exit.  To solve the task,
the agent must find the beacon, observe and remember its color, and use it to
identify and reach the \emph{good} exit which has the same color.  The
positions of the agent, the exits, and the beacon, as well as the colors of the
exits and beacons are randomly sampled such that each episode is unique.
\Cref{fig:screenshot:memoryroomsseven,fig:screenshot:memoryroomsnine} shows
state and observation frames respectively taken from instances of
\memoryroomsseven\ and \memoryroomsnine.

\paragraph{States and Observations}
For \memoryroomsseven, the state \emph{grid} component is a $3\times 7\times 7$
volume, while the observation \emph{grid} component is a $3\times 2\times 3$
volume representing a $2\times 3$ view of the agent surroundings.  For
\memoryroomsnine, the state \emph{grid} component is a $3\times 9\times 9$
volume.

\paragraph{Actions}
Each time-step, the agent must choose an action from the set \{
\textbf{MOVE\_FORWARD}, \textbf{MOVE\_BACKWARD}, \textbf{MOVE\_LEFT},
\textbf{MOVE\_RIGHT}, \textbf{TURN\_LEFT}, \textbf{TURN\_RIGHT},
\textbf{PICK\_N\_DROP}, \textbf{ACTUATE} \}.  The \textbf{MOVE\_*} actions
result in a movement depending on the agent's orientation, while the
\textbf{TURN\_*} allows the agent to change its orientation.  With the
\textbf{PICK\_N\_DROP} action, the agent can pick and/or drop the key from/to
the cell in front, while with the \textbf{ACTUATE} action, the agent can open
and/or close doors.  \memoryroomsseven\ and \memoryroomsnine have no doors or
pickable items, therefore the last two actions have no effect.

\paragraph{Rewards}
The agent receives a dense reward signal composed as the sum of the following
terms:
\begin{itemize}
  \item a living reward of $-0.05$ for every time-step;
  \item a reward of $5.0$ for reaching the \emph{good} exit;
  \item a reward of $-5.0$ for reaching the \emph{bad} exit.
\end{itemize}

\section{Model Architectures}\label{sec:architectures}

\begin{figure*}
  \centering
  \begin{subfigure}{.3\linewidth}
    \centering
    \includegraphics[width=\linewidth]{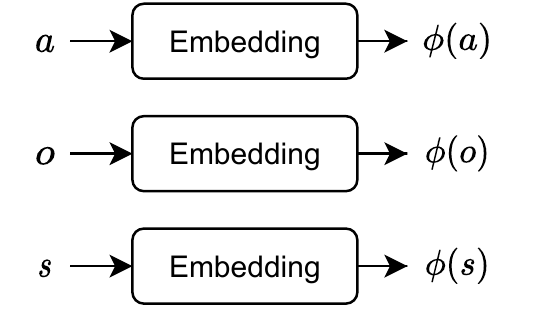}
    \caption{State, action, and observation representation models used for
    \emph{flat} POMDPs.}\label{fig:architecture:representations:flat}
  \end{subfigure}
  \qquad\qquad
  \begin{subfigure}{.5\linewidth}
    \centering
    \includegraphics[width=\linewidth]{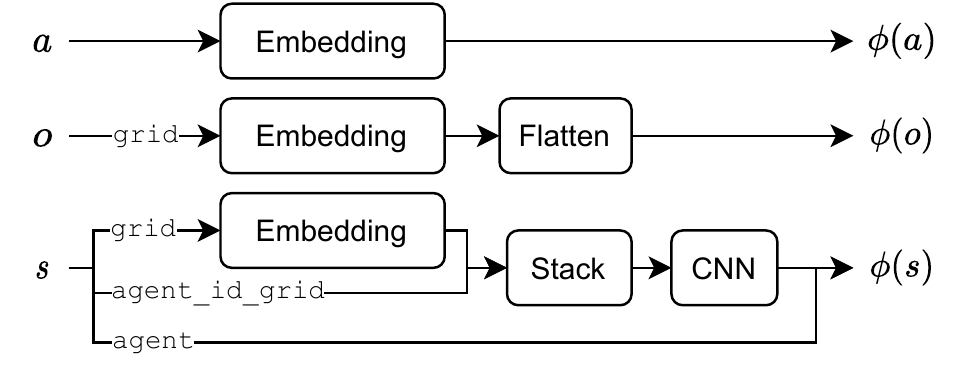}
    \caption{State, action, and observation representation models used for
    \emph{gridverse} POMDPs.}\label{fig:architecture:representations:gridverse}
  \end{subfigure}

  \begin{subfigure}{.8\linewidth}
    \centering
    \includegraphics[width=\linewidth]{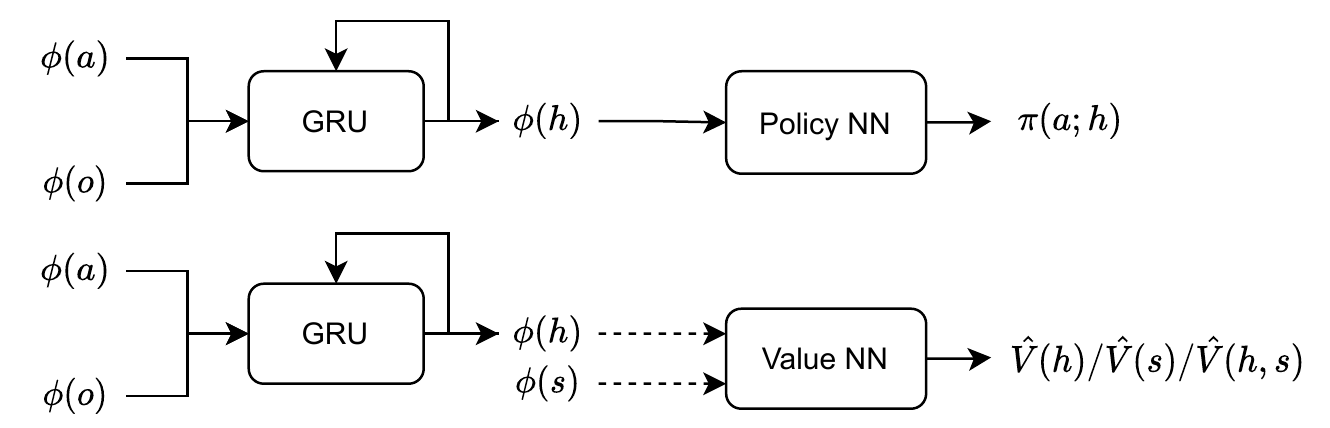}
    \caption{A2C architecture.  Separate components are used for policy and
      critic models.  For the \emph{categorical} and \emph{gridverse}
      environments, state, action, and observation representations $\phi(s)$,
      $\phi(a)$, and $\phi(o)$ are shown in
      \Cref{fig:architecture:representations:flat} and
      \Cref{fig:architecture:representations:gridverse}.  For the
      \emph{feature-vector} environments, states and observation
      representations $\phi(s)$ and $\phi(o)$ are directly provided by the
      environment directly, while action representations $\phi(a)$ are
    implemented as embeddings.  Dotted lines are present or omitted depending
  on whether a history critic $\vmodel(h)$, state critic $\vmodel(s)$, or
history-state critic $\vmodel(h, s)$ is being
modeled.}\label{fig:architecture:a2c}
  \end{subfigure}
  \caption{}\label{fig:architecture}
\end{figure*}

In this section, we describe the architectures used by the policy and critic
models for each environment, also shown in \Cref{fig:architecture}.  The
general architecture will be similar for all domains;  however there will some
differences to accommodate the different state and observation representations
provided by each environment.

\paragraph{General Architecture}

These components are shown in \Cref{fig:architecture:a2c}.  Policy and critic
models share part of the architecture, but not the associated parameters.
Concatenated action and observation features form the input to a
$128$-dimensional single-layer \emph{gated recurrent unit}
(GRU)~\cite{cho_properties_2014}, which acts as a history representation
$\phi(h)$.  The policy and value NN components vary in each environment. 

\paragraph{Categorical POMDPs}

The action, state, and observation feature components are shown in
\Cref{fig:architecture:representations:flat}.  Because categorical POMDPs
provide states, actions and observations as categorical indices, we use
$64$-dimensional embedding models to represent each of them.  
The policy and value NN components are each $2$-layer feedforward models with
$512$ and $256$ nodes with ReLU non-linearities.

\paragraph{Feature-Vector POMDPs}

The action featues are one-hot encodings of the respective categorical indices.
Because the states and observations provided by the environments already come
in a simple and flat feature representation, we do not process them further.
The policy and value NN components are each $2$-layer feedforward models with
$512$ and $256$ nodes with ReLU non-linearities.

\paragraph{Gridverse POMDPs}

These feature extraction components are shown in
\Cref{fig:architecture:representations:gridverse}.  Because gridverse POMDPs
provide actions as categorical indices, we use $1$-dimensional embedding models
to represent them, i.e., we focus on the state and observation representations
as they contain the most relevant information.  On the other hand, states and
observations are provided in the format described in
\Cref{sec:environments:gridverse}.
The $3\times 2\times 3$ observations are first processed using an
$8$-dimensional embedding layer, and then flattened, which produces a
$144$-dimensional observation feature $\phi(o)$.
The states contain relevant information in different forms, and require a more
complex model.  The \texttt{grid} component is processed using an embedding
layer, which is then stacked with the \texttt{agent\_id\_grid} component, and
processed by a $3$-layer convolutional network.  The output of the
convolutional layers is concatenated with the \texttt{agent} components, to
form the overall state feature $\phi(s)$.
The policy and value NN components are each single-layer feedforward models
with $512$ nodes with ReLU non-linearities.

\section{Hyperparameters and Grid Search}\label{sec:hpsearch}

For each environment and method, we perform a separate hyper-parameter grid
search to find the respective best training performance possible.  In each
case, the grid search is performed over the following hyper-parameters and
ranges of values:
\begin{enumerate}
  \item the actor learning rate $\alpha_{\policy}$, searched over values 0.0001,
    0.0003, and 0.001,
  \item the critic learning rate $\alpha_{\vmodel}$, searched over values 0.0001,
    0.0003, and 0.001,
  \item the initial negative-entropy weight $\lambda_0$, searched over
    environment-dependent values:
    \begin{description}
      \item[\heavenhellthree] 0.01, 0.03, 0.1, 0.3, 1.0.
      \item[\heavenhellfour] 0.01, 0.03, 0.1, 0.3, 1.0.
      \item[\shoppingfive] 0.3, 1.0, 3.0, 10.0, 30.0.
      \item[\shoppingsix] 0.3, 1.0, 3.0, 10.0, 30.0.
      \item[\carflag] 0.03, 0.1, 0.3, 1.0, 3.0.
      \item[\cleaner] 0.03, 0.1, 0.3, 1.0, 3.0.
      \item[\memoryroomsseven] 0.01, 0.03, 0.1, 0.3, 1.0.
      \item[\memoryroomsnine] 0.01, 0.03, 0.1, 0.3, 1.0.
    \end{description}
\end{enumerate}

In total, a full grid search over these hyper-parameters amounts to $3\cdot
3\cdot 5 = 45$ different hyper-parameter possibilities for each environment and
method.  Factoring in the $20$ independent runs, the $5$ methods, and $8$
environments, this adds up to $45\cdot 20\cdot 5\cdot 8 = 36$k separate runs.
The optimal hyperparameters for each case are shown in \Cref{tab:hparams}.
Other relevant hyper-parameters are set as follows:
\begin{itemize}
  \item The negative-entropy weight decays linearly over the course of $2$M
    timesteps to a final value equal one tenth of the initial one,
    $\frac{\lambda_0}{10}$.
  \item The number of episodes sampled per gradient step ($E$ in
    \Cref{alg:code}) is set to $2$.
  \item Episodes are automatically terminated if they do not end after $100$
    timesteps.
  \item A frozen target model is used to stabilize the training of critics,
    with the target model parameters being updated every $10$k timesteps.
\end{itemize}

\begin{table}
  \centering
  \caption{Hyperparameter grid search results.}\label{tab:hparams}
  \begin{tabular}{lllll}
    \toprule
    Domain & Method & $\alpha_{\policy}$ & $\alpha_{\vmodel}$ & $\lambda_0$ \\
    \midrule
    \heavenhellthree & \achs & 0.001 & 0.001 & 0.1 \\
    & \acs & 0.001 & 0.001 & 1.0 \\
    & \ach & 0.001 & 0.001 & 0.1 \\
    & \achrtwo & 0.001 & 0.0003 & 1.0 \\
    & \achrfour & 0.001 & 0.0003 & 1.0 \\
    \midrule
    \heavenhellfour & \achs & 0.001 & 0.001 & 0.1 \\
    & \acs & 0.001 & 0.001 & 0.1 \\
    & \ach & 0.001 & 0.0003 & 0.3 \\
    & \achrtwo & 0.001 & 0.0003 & 0.3 \\
    & \achrfour & 0.001 & 0.0003 & 0.3 \\
    \midrule
    \shoppingfive & \achs & 0.001 & 0.0003 & 3.0 \\
    & \acs & 0.001 & 0.001 & 10.0 \\
    & \ach & 0.001 & 0.0003 & 3.0 \\
    & \achrtwo & 0.001 & 0.001 & 3.0 \\
    & \achrfour & 0.001 & 0.001 & 3.0 \\
    \midrule
    \shoppingsix & \achs & 0.001 & 0.0003 & 3.0 \\
    & \acs & 0.001 & 0.001 & 10.0 \\
    & \ach & 0.001 & 0.0003 & 3.0 \\
    & \achrtwo & 0.001 & 0.001 & 1.0 \\
    & \achrfour & 0.001 & 0.0003 & 10.0 \\
    \midrule
    \carflag & \achs & 0.001 & 0.001 & 0.03 \\
    & \acs & 0.001 & 0.001 & 0.03 \\
    & \ach & 0.001 & 0.001 & 0.03 \\
    & \achrtwo & 0.001 & 0.001 & 0.03 \\
    & \achrfour & 0.001 & 0.001 & 0.03 \\
    \midrule
    \cleaner & \achs & 0.001 & 0.001 & 1.0 \\
    & \acs & 0.001 & 0.001 & 1.0 \\
    & \ach & 0.001 & 0.001 & 1.0 \\
    & \achrtwo & 0.001 & 0.001 & 3.0 \\
    & \achrfour & 0.001 & 0.001 & 1.0 \\
    \midrule
    \memoryroomsseven & \achs & 0.0003 & 0.001 & 0.1 \\
    & \acs & 0.0003 & 0.001 & 0.01 \\
    & \ach & 0.0003 & 0.0001 & 0.1 \\
    & \achrtwo & 0.001 & 0.001 & 0.3 \\
    & \achrfour & 0.001 & 0.001 & 0.3 \\
    \midrule
    \memoryroomsnine & \achs & 0.001 & 0.0003 & 0.3 \\
    & \acs & 0.001 & 0.0003 & 0.1 \\
    & \ach & 0.001 & 0.0003 & 0.3 \\
    & \achrtwo & 0.001 & 0.001 & 0.1 \\
    & \achrfour & 0.001 & 0.001 & 0.1 \\
    \bottomrule
  \end{tabular}
\end{table}

%%%%%%%%%%%%%%%%%%%%%%%%%%%%%%%%%%%%%%%%%%%%%%%%%%%%%%%%%%%%%%%%%%%%%%%%

\end{document}